\documentclass[10pt]{amsart}
\usepackage[T1]{fontenc}
\usepackage{palatino}
\usepackage{amsmath,amssymb,amsthm}
\usepackage{microtype}
\usepackage{graphicx}
\usepackage{hyperref}
\usepackage{algorithm,algorithmic}
\def\p{\mathrm{P}}

\def\rnk{\mathrm{rank}}
\def\G{\Gamma}
\newcommand{\B}{\sum_{(k,i)\in E} (e^{M}_i \otimes I_d){x_{k,i}}e_{ki}^T  }

\newcommand{\D}{\sum_{(k,i)\in E}(e^{M}_i \otimes I_d) x_{k,i}{x_{k,i}}^T (e^{M}_i \otimes I_d)^T }
\newcommand{\Tr}{\mathrm{Tr}}
\newcommand{\Zs}{Z^\star}
\newcommand{\Ws}{W^\star}
\newcommand{\Os}{O^{\star}}
\newcommand{\Zss}{Z^{\star \star}}
\newcommand{\Oss}{O^{\star \star}}

\newcommand{\Gs}{G^\star}
\newcommand{\defn}{\stackrel{\text{def}}{=}}

\newcommand{\vertiii}[1]{{\left\vert\kern-0.25ex\left\vert\kern-0.25ex\left\vert #1
\right\vert\kern-0.25ex\right\vert\kern-0.25ex\right\vert}}

\newtheorem{theorem}{Theorem}
\newtheorem{definition}[theorem]{Definition}
\newtheorem{observation}[theorem]{Observation}
\newtheorem{proposition}[theorem]{Proposition}
\newtheorem{corollary}[theorem]{Corollary}

\newtheorem{lemma}[theorem]{Lemma}

\title[Global registration using semidefinite programming]{Global registration of multiple point clouds using semidefinite programming}

 \author{Kunal N. Chaudhury}
 \address{Department of Electrical Engineering, Indian Institute of Science, Bangalore 560012}
 \email{kunal@ee.iisc.ernet.in}
 
 \author{Yuehaw Khoo}
 \address{Department of Physics, Princeton University, Princeton, NJ 08540}
 \email{ykhoo@princeton.edu}
 
 \author{Amit Singer}
 \address{Department of Mathematics and PACM, Princeton University, Princeton, NJ 08540}
 \email{amits@math.princeton.edu}

\begin{document}
\maketitle
\begin{abstract} Consider $N$ points in $\mathbb{R}^d$ and $M$ local coordinate systems that are related through unknown rigid transforms.
For each point we are given (possibly noisy) measurements of its local coordinates in some of the coordinate systems.
Alternatively, for each coordinate system, we observe the coordinates of a subset of the points.
The problem of estimating the global coordinates of the $N$ points (up to a rigid transform) from such measurements comes up in distributed approaches
to molecular conformation and sensor network localization, and also in computer vision and graphics.

The least-squares formulation of this problem, though non-convex, has a well known closed-form solution when $M=2$ (based on the singular value decomposition). However, no closed form solution is known for $M\geq 3$.

In this paper, we demonstrate how the least-squares formulation can be relaxed into a convex program, namely a semidefinite program (SDP).
By setting up connections between the uniqueness of this SDP and results from rigidity theory, we prove conditions for exact and stable recovery for the SDP relaxation.
In particular, we prove that the SDP relaxation can guarantee recovery under more adversarial conditions compared to earlier proposed spectral relaxations, and derive error bounds for the registration error
incurred by the SDP relaxation.

We also present results of numerical experiments on simulated data to confirm the theoretical findings. We empirically demonstrate that (a) unlike the spectral relaxation, the relaxation gap is mostly zero for the semidefinite
program (i.e., we are able to solve the original non-convex least-squares problem) up to a certain noise threshold, and (b) the semidefinite program performs significantly better than spectral and manifold-optimization methods, particularly
at large noise levels.

\end{abstract}


\section{Introduction}

The problem of point-cloud registration comes up in computer vision and graphics \cite{Sharp,tzeneva2011global,859274}, and in distributed approaches to molecular conformation \cite{fang2013disco,cucuringu2013asap} and
sensor network localization \cite{cucuringu2012sensor,biswas2008distributed}. The registration problem in question is one of determining the coordinates of a point cloud
$P$ from the knowledge of (possibly noisy) coordinates  of smaller point cloud subsets (called \textit{patches}) $P_1,\ldots,P_M$ that are derived from $P$ through some general transformation.
In certain applications  \cite{mitra2004registration,tzeneva2011global,GlobalRegKrishnan}, one is often interested in finding the optimal transforms (one for each patch)
that consistently align $P_1,\ldots,P_M$. This can be seen as a sub-problem in the determination of the coordinates of $P$ \cite{cucuringu2012sensor,rusinkiewicz2001efficient}.

In this paper, we consider the problem of \textit{rigid registration} in which the points within a given $P_i$ are (ideally) obtained from $P$ through an unknown rigid transform. Moreover, we assume that the correspondence
between the local patches and the original point cloud is known, that is, we know beforehand as to which points from $P$ are contained in a given $P_i$. In fact, one has a control on the
correspondence in distributed approaches to molecular conformation \cite{cucuringu2013asap} and sensor network localization \cite{biswas2008distributed,zhang2010arap,cucuringu2012sensor}.
While this correspondence is not directly available for certain graphics and vision problems, such as multiview registration \cite{pottmann2006geometry}, it is in principle
possible to estimate the correspondence by aligning pairs of patches, e.g., using the ICP (Iterative Closest Point) algorithm \cite{besl1992method,rusinkiewicz2001efficient,huang2013consistent}.

\subsection{Two-patch Registration} The particular problem of two-patch registration has been well-studied \cite{faugeras1986representation,horn1987closed,arun1987svd}. In the noiseless setting, we are
given two point clouds $\{x_1,\ldots,x_N\}$ and $\{y_1,\ldots,y_N\}$ in $\mathbb{R}^d$, where the latter is obtained through some rigid transform of the former. Namely,
\begin{equation}
\label{model_2patch}
y_k = O x_k + t \qquad (k = 1,\ldots,N),
\end{equation}
where $O$ is some unknown $d\times d$ orthogonal matrix  (that satisfies $O^TO=I_d$)  and $t \in \mathbb{R}^d$ is some unknown translation.

The problem is to infer $O$ and $t$ from the above equations. To uniquely determine $O$ and $t$, one must have at least $N \geq d+1$ non-degenerate
points\footnote{By non-degenerate, we mean that the affine span of the points is $d$ dimensional.}. In this case, $O$ can be determined
simply by fixing the first equation in \eqref{model_2patch} and
subtracting (to eliminate $t$) any of the remaining $d$ equations from it. Say, we subtract the next $d$ equations:
\begin{equation*}
 [y_2 - y_1 \ \cdots \ y_{d+1} - y_1 ] = O  [x_2 - x_1 \ \cdots \ x_{d+1} - x_1 ].
\end{equation*}
By the non-degeneracy assumption, the matrix on the right of $O$ is invertible, and this gives us $O$. Plugging $O$ into any of the equations in \eqref{model_2patch}, we get $t$.

In practical settings, \eqref{model_2patch} would hold only approximately, say, due to noise or model imperfections.
A particular approach then would be to determine the optimal $O$ and $t$ by considering  the following least-squares program:
\begin{equation}
\label{costArun}
\min_{O \in \mathbb{O}(d), \ t \in \mathbb{R}^d} \quad  \sum_{k=1}^N \ \lVert  y_k - O x_k - t  \rVert_2^2.
\end{equation}
Note that the problem looks difficult a priori since the domain of optimization is $\mathbb{O}(d) \times  \mathbb{R}^d$, which is non-convex. Remarkably, the global minimizer of this
non-convex problem can be found exactly, and has a simple closed-form expression \cite{fan1955some,keller1975closest,higham1986,faugeras1986representation,horn1987closed,arun1987svd}.
More precisely, the optimal $\Os$ is given by $VU^T$, where $U\Sigma V^T$ is the singular value decomposition (SVD) of
\begin{equation*}
\sum_{k=1}^N  (x_k - x_c) (y_k - y_c)^T,
\end{equation*}
in which $x_c=(x_1 + \cdots + x_N)/N $ and $y_c=(y_1 + \cdots + y_N)/N$ are the centroids of the respective point clouds. The optimal translation is $t^{\star} = y_c - \Os x_c$.

The fact that two-patch registration has a closed-form solution is used in the so-called incremental (sequential) approaches for registering multiple patches \cite{besl1992method}.
The most well-known method is the ICP algorithm \cite{rusinkiewicz2001efficient} (note that ICP uses other heuristics and refinements besides registering  corresponding points). Roughly, the idea in sequential registration  is
to register two overlapping patches at a time, and then integrate the estimated pairwise transforms using some means. The integration can be achieved either locally (on a patch-by-patch basis),
or using global cycle-based methods such as synchronization \cite{Sharp,howard2010estimation,singer2011angular,tzeneva2011global,LUDLanhui}. More recently, it was demonstrated that, by locally registering overlapping patches and then integrating the pairwise transforms
using synchronization, one can design efficient and robust methods for distributed sensor network localization \cite{cucuringu2012sensor} and molecular conformation \cite{cucuringu2013asap}.
 Note that, while the registration phase is local, the synchronization method integrates the local transforms in a globally consistent manner. This makes it robust to error propagation that often plague
 local integration methods \cite{howard2010estimation,LUDLanhui}.

\subsection{Multi-patch Registration} To describe the multi-patch registration problem, we first introduce some notations. Suppose $x_1,x_2,\ldots,x_N$ are the unknown global coordinates of a
point cloud in $\mathbb{R}^d$. The point cloud is divided into patches $P_1,P_2,\ldots,P_M$, where each
$P_i$ is a subset of $\{x_1,x_2,\ldots,x_N\}$. The patches are in general overlapping, whereby a given point can belong to multiple patches.
We represent this membership using an undirected bipartite graph $\G=(V_x \cup V_P, E)$. The set of vertices $V_x= \{x_1,\ldots,x_N\}$ represents the point cloud, while $V_P= \{P_1,\ldots,P_M\}$
represents the patches. The edge set $E = E(\G)$ connects $V_x$ and $V_P$, and is given by the requirement that $(k,i) \in E$ if and only if $x_k \in P_i$.
We will henceforth refer to $\G$ as the \textit{membership graph}.

\begin{figure}[here]
\centering
\includegraphics[width=1.0\linewidth]{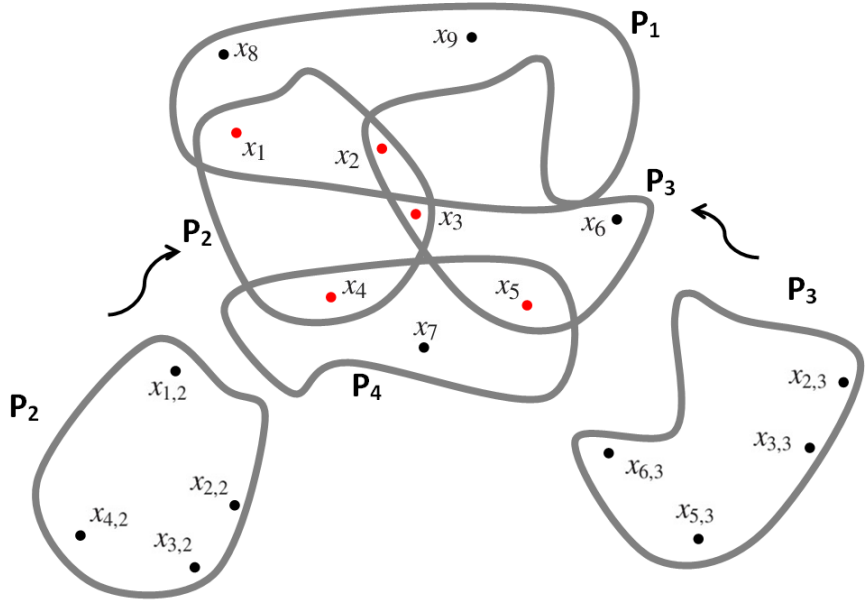}
\caption{The problem of registering $3$ patches on $\mathbb{R}^2$, where one is required to find the global coordinates of the points from the corresponding local patch coordinates.
The local coordinates of the points in patches $P_2$ and $P_3$ are shown (see \eqref{data} for the notation of local coordinates). It is the common points belonging to two or more patches (marked in red) that
contribute to the registration. Note that sequential or pairwise registration would fail in this case. This is because no pair of patches can be registered as they have less than $3$ points in common (at least 3 points are required
to fix rotations, reflections, and translations in $\mathbb{R}^2$).  The SDP-based algorithm proposed in this paper does a global registration, and is able to recover the exact global coordinates
for this example.}
\label{3P}
\end{figure}

In this paper, we assume that the local coordinates of a given patch can (ideally)  be related to the global coordinates through a single rigid transform, that is, through some rotation, reflection, and translation.
More precisely, with every patch $P_i$ we associate some (unknown) orthogonal transform $O_i$ and translation $t_i$. If point $x_k$ belongs to patch $P_i$, then its representation in $P_i$ is given
by (cf. \eqref{model_2patch} and Figure \ref{3P})
\begin{equation}
 x_{k,i}=O_i^T(x_k - t_i) \qquad (k,i) \in E(\G).
\end{equation}
Alternatively, if we fix a particular patch $P_i$, then for every point belonging to that patch,
\begin{equation}
\label{exact}
x_k = O_i x_{k,i} + t_i \qquad (k,i) \in E(\G).
\end{equation}
In particular, a given point can belong to multiple patches, and will have a different representation in the
coordinate system of each  patch.

The premise of this paper is that we are given the membership graph and the local coordinates (referred to as measurements), namely
\begin{equation}
\label{data}
\G  \quad  \text{and} \quad \{x_{k,i}, (k,i) \in E(\G)\},
\end{equation}
and the goal is to recover the coordinates $x_1,\ldots,x_N$, and in the process the unknown rigid transforms $(O_1,t_1),\ldots,(O_M,t_M)$, from \eqref{data}.
Note that the global coordinates are determined up to a global rotation, reflection, and translation. We say that two points clouds (also referred to as \textit{configurations}) are \textit{congruent} if one is obtained through a rigid
transformation of the other. We will always identify two congruent configurations as being a single configuration.

Under appropriate non-degeneracy assumptions on the measurements, one task would be to specify appropriate conditions on $\G$ under which the global coordinates can be uniquely determined.
Intuitively, it is clear that the patches must have enough points in common for the registration problem to have an unique solution.
For example, it is clear that the global coordinates cannot be uniquely recovered if $\G$ is disconnected.

In practical applications, we are confronted with noisy settings where  \eqref{exact} holds only approximately.
In such cases, we would like to determine the global coordinates and the rigid
transforms such that the discrepancy  in \eqref{exact} is minimal. In particular, we consider the following quadratic loss:
\begin{equation}
\label{obj}
\phi=  \sum_{(k,i) \in E(\G)} \ \lVert x_k - O_i x_{k,i} - t_i \rVert^2,
\end{equation}
where $\lVert \cdot \rVert$ is the Euclidean norm on $\mathbb{R}^d$. The optimization problem is to minimize $\phi$ with respect to the following variables:
\begin{equation*}
 x_1,x_2,\ldots,x_N \in \mathbb{R}^d, \quad O_1,\ldots,O_M \in \mathbb{O}(d), \quad t_1,\ldots,t_M \in \mathbb{R}^d.
\end{equation*}
The input to the problem are the measurements in \eqref{data}. Note that our ultimate goal is to determine $x_1,x_2,\ldots,x_N$; the rigid transforms can be seen as latent variables.

The problem of multipatch registration is intrinsically non-convex since one is required to optimize over the non-convex domain of orthogonal transforms.
Different ideas from the optimization literature have been deployed to attack this problem, including Lagrangian optimization and projection
methods. In the Lagrangian setup, the orthogonality constraints are incorporated into the objective; in the projection method, the constraints are forced after every step of the optimization \cite{pottmann2006geometry}.
Following the observation that the registration problem can be viewed as an optimization on the Grassmanian and Stiefel manifolds, researchers have proposed algorithms using ideas from the theory
and practice of manifold optimization \cite{GlobalRegKrishnan}. A detailed review of these methods is beyond the scope of this paper, and instead we refer the interested reader to these excellent reviews \cite{edelman1998geometry,absil2009optimization}.
Manifold-based methods are, however, local in nature, and are not guaranteed to find the global minimizer. Moreover, it is rather difficult to certify the noise stability of such methods.

\subsection{Contributions}

The main contributions of the paper can be organized into the following categories. \newline

\begin{enumerate}
\item \textbf{Algorithm}: We demonstrate how the translations can be factored out of \eqref{obj}, whereby the least-squares problem can be reduced to the following optimization:
\begin{equation}
\label{MC}
 \max_{O_1,\ldots,O_M} \ \sum_{i,j=1}^M \Tr(O_i C_{ij} O_j^T ) \quad \text{ subject to } \quad  O_1,\ldots,O_M \in \mathbb{O}(d),
\end{equation}
where  $C_{ij} \in \mathbb{R}^{d \times d} (1 \leq i,j \leq M)$ are the  $(i,j)$-th sub-blocks of some  positive semidefinite block matrix $C$ of size $Md \times Md$.
Given the  solution of \eqref{MC}, the desired global coordinates can simply be obtained by solving a linear system.
It is virtually impossible to find the global optimum of \eqref{MC} for large-scale problems ($M \gg 1$), since this involves the optimization of a quadratic cost on a huge non-convex parameter space. 
In fact, the simplest case $d=1$ with $C$ as the Laplacian matrix corresponds to the \texttt{MAX-CUT} problem, which is known to be NP-hard. 
The main observation of this paper is that \eqref{MC} can instead be  relaxed into a convex program, namely a semidefinite program, whose global optimum can be approximated to any finite precision in polynomial time using standard off-the-shelf solvers.
This yields a tractable method for global registration described in Algorithm \ref{algo:GRET}.
The corresponding algorithm derived from the spectral relaxation of \eqref{MC} that was already considered in \cite{GlobalRegKrishnan} is described in Algorithm \ref{algo:SPEC} for reference.
\newline

\item \textbf{Exact Recovery}: We present conditions on the coefficient matrix $C$ in \eqref{MC} for exact recovery using Algorithm \ref{algo:GRET}.
In particular, we show that the exact recovery questions about Algorithm \ref{algo:GRET} can be mapped into rigidity theoretic questions that have already been investigated
earlier\footnote{The authors thank the anonymous referees for pointing this out.} in \cite{zha2009spectral,gortler2013affine}.
The contribution of this section is the connection made between the $C$ matrix in \eqref{MC} and various notions of rigidity considered in these papers.
We also present an efficient randomized rank test for $C$ than can be used to certify exact recovery (motivated by the work in \cite{hendrickson1992conditions,gortler2010characterizing,MatCompletionSinger}).
\newline

\item \textbf{Stability Analysis}:  We study the stability of Algorithms \ref{algo:SPEC} and \ref{algo:GRET} for the noise model in which the patch coordinates are perturbed using noise of bounded size (note that the stability of the
spectral relaxation was not investigated in \cite{GlobalRegKrishnan}). Our main result here is Theorem \ref{thm:stability2} which states that, if $C$ satisfies a particular rank condition,
then the registration error for Algorithm  \ref{algo:GRET} is within a constant factor of the noise level. To the best of our knowledge, there is no existing algorithm for multipatch registration that comes with a
similar stability guarantee.
\newline

\item \textbf{Empirical Results}: We present numerical results on simulated data to numerically verify the exact recovery and noise stability properties of Algorithms \ref{algo:SPEC} and \ref{algo:GRET}.
Our main empirical findings are the following: \\~\\
(1) The semidefinite relaxation performs significantly better than
spectral and manifold-based optimization (say, with the spectral solution as initialization) in terms of  the reconstruction quality (cf. first plot in Figure \ref{fig:stability1}).   \\~\\
(2) Up to a certain noise level,  we are actually able to solve the original non-convex problem using the semidefinite relaxation (cf. second plot in Figure \ref{fig:stability1}).

\end{enumerate}

\subsection{Broader Context and Related Work}

The objective  \eqref{obj} is a straightforward extension of the objective for two-patches \cite{fan1955some,faugeras1986representation,horn1987closed,arun1987svd}. In fact, this objective was earlier considered
by Zhang et al. for distributed sensor localization \cite{zhang2010arap}. The present work is also closely tied to the work of Cucuringu et al. on distributed localization
\cite{cucuringu2012sensor,cucuringu2013asap}, where a similar objective is implicitly optimized. The common theme in these works is that some form of optimization is used to
globally register the patches, once their local coordinates have been determined by some means. There is, however, some fundamental differences between the various algorithms used to actually perform the optimization.
Zhang et al. \cite{zhang2010arap} use alternating least-squares to iteratively optimize over the global coordinates and the transforms, which to the best of our knowledge has no convergence guarantee.
On the other hand, Cucuringu et al. \cite{cucuringu2012sensor,cucuringu2013asap} first optimize over the orthogonal transforms (using synchronization \cite{singer2011angular}), and then solve
for the translations (in effect, the global coordinates) using least-squares fitting. In this work, we combine these different ideas into a single framework.
While our objective is similar to the one used in \cite{zhang2010arap}, we jointly optimize the rigid transforms and positions.
In particular, the algorithms considered in Section \ref{Opt} avoid the convergence issues associated with alternating least-squares in \cite{zhang2010arap}, and are
able to register patch systems that cannot be registered using the approach in \cite{cucuringu2012sensor,cucuringu2013asap}.

Another closely related work is the paper by Krishnan et al. on global registration \cite{GlobalRegKrishnan}, where the optimal transforms (rotations to be specific) are computed by
extending the objective in \eqref{costArun} to the multipatch case. The subsequent mathematical formulation has strong resemblance
with our formulation, and, in fact, leads to a subproblem similar to \eqref{MC}.
Krishnan et al. \cite{GlobalRegKrishnan} propose the use of manifold optimization to solve \eqref{MC}, where the manifold is the product manifold of rotations. However, as mentioned earlier, manifold methods generally do
not offer guarantees on convergence (to the global minimum) and stability. Moreover, the manifold in \eqref{MC} is not connected. Therefore, \textit{any} local method will fail to attain the global optimum of \eqref{MC} if the
initial guess is on the wrong component of the manifold.

It is exactly at this point that we depart from \cite{GlobalRegKrishnan}, namely, we propose to relax \eqref{MC} into a tractable semidefinite program (SDP).
This was motivated by a long line of work on the use of  SDP relaxations for non-convex (particularly NP-hard) problems. See,
for example, \cite{lovasz1991cones,goemans1995improved,wolkowicz2002semidefinite,nemirovski2007sums,candes2009exact,lerman2012robust}, and these reviews \cite{vandenberghe1996semidefinite,nesterov1998semidefinite,5447068}.
Note that for $d=1$,  \eqref{MC} is a quadratic Boolean optimization, similar to the \texttt{MAX-CUT} problem.
An SDP-based algorithm with randomized rounding  for solving \texttt{MAX-CUT} was proposed
in the seminal work of Goemans and Williamson \cite{goemans1995improved}. The semidefinite relaxation that we consider in Section \ref{Opt} is motivated by this work.
In connection with the present work, we note that provably stable SDP algorithms have been considered for low rank matrix completion \cite{candes2009exact}, phase retrieval \cite{candes2012phaselift,waldspurger2012phase},
and graph localization \cite{Montanari2012}.

We note that a special case of the registration problem considered here is the so-called generalized Procrustes problem \cite{gower2004procrustes}. Within the point-patch framework just introduced, the goal
in Procrustes analysis is to find $O_1,\ldots,O_M  \in  \mathbb{O}(d)$ that minimize
\begin{equation}
\label{proc}
  \sum_{k=1}^N \sum_{i,j=1}^M \ \lVert O_i x_{k,i} - O_j x_{k,j} \rVert^2.
\end{equation}
In other words, the goal is to achieve the best possible alignment of the $M$ patches through orthogonal transforms. This can be seen as an instance of the global registration problem without the translations ($t_1=\cdots=t_M=0$),
and in which $\Gamma$ is complete. It is not difficult to see that \eqref{proc} can be reduced to  \eqref{MC}. On the other hand, using the analysis in Section \ref{Opt}, it can be shown that \eqref{obj} is equivalent to
\eqref{proc} in this case. While the Procrustes problem is known to be NP-hard, several polynomial-time approximations with  guarantees have been proposed.
In particular, SDP relaxations of \eqref{proc} have been considered in \cite{nemirovski2007sums,so2011moment,naor2013efficient}, and more recently in \cite{Ortho-Cut}.
We use the  relaxation of \eqref{MC} considered in \cite{Ortho-Cut} for reasons to be made precise in Section \ref{Opt}.

\subsection{Notations} We use upper case letters such as $O$ to denote matrices, and lower case letters such as $t$ for vectors. We use $I_d$ to denote the identity matrix of size $d \times d$. We denote the diagonal matrix of size $n \times n$ with diagonal elements $c_1,\ldots,c_n$ as $\mathrm{diag}(c_1,\ldots,c_n)$.
We will frequently use block matrices built from smaller matrices, typically of size $d \times d$, where $d$ is
the dimension of the ambient space. For some block matrix $A$, we will use $A_{ij}$ to denote its $(i,j)$-th block, and $A(p,q)$ to denote its $(p,q)$-th entry. In particular, if each block has size $d \times d$, then
\begin{equation*}
A_{ij}(p,q) = A\big((i-1)d + p, (j-1)d + q \big)  \qquad (1 \leq p,q\leq d).
\end{equation*}
We use $A \succeq 0$ to mean that $A$ is positive semidefinite, that is, $u^T A u \geq 0$ for all $u$.
We use $\mathbb{O}(d)$ to denote the group of orthogonal transforms (matrices) acting on $\mathbb{R}^d$, and $\mathbb{O}(d)^M$ to denote the $M$-fold product of $\mathbb{O}(d)$ with itself.
We will also conveniently identify the matrix $[O_1 \cdots O_M]$ with an element of $\mathbb{O}(d)^M$ where each $O_i \in \mathbb{O}(d)$. We use $\lVert x \rVert$ to denote the
Euclidean norm of $x \in \mathbb{R}^n$ ($n$ will usually be clear from the context, and will be pointed
out if this is not so). We denote the trace of a square matrix $A$ by $\Tr(A)$. The Frobenius  and spectral norms are defined as
\begin{equation*}
\lVert A \rVert_F = \Tr(A^T A)^{1/2} \quad \text{and} \quad \|A\|_{\mathrm{sp}} = \max_{\|x\| \leq 1} \ \|Ax\|.
\end{equation*}
The Kronecker product between matrices $A$ and $B$ is denoted by $A \otimes B$ \cite{Golub}. The all-ones vector is denoted by $e$ (the dimension will be obvious from the context), and
$e^N_i$ denotes the all-zero vector of length $N$ with $1$ at the $i$-th position.

\subsection{Organization} In the next section, we present the semidefinite relaxation of the least-squares registration problem described in the introduction.
For reference, we also present the closely related spectral relaxation that was already considered in \cite{GlobalRegKrishnan,zha2009spectral,gortler2013affine}.
Exact recovery questions are addressed in section \ref{ER}, followed by a randomized test in section \ref{randTest}.
Stability analysis for the spectral and semidefinite relaxations are presented in section \ref{Stab}. Numerical simulations can be found in section \ref{Sim}, and a discussion of
certain open questions in section \ref{Conc}.

\section{Spectral and Semidefinite Relaxations}
\label{Opt}

The minimization of \eqref{obj} involves unconstrained variables (global coordinates and patch translations) and constrained variables (the orthogonal transformations).
We  first solve for the unconstrained variables in terms of the unknown orthogonal transformations, representing the former as linear combinations of the latter.
This reduces \eqref{obj} to a quadratic optimization problem over the orthogonal transforms of the form \eqref{MC}.

In particular, we combine the global coordinates and the translations into a single matrix:
\begin{equation}
\label{Z}
Z = \big[x_1 \ \cdots \ x_N \ t_1 \ \cdots \ t_M \big] \in \mathbb{R}^{d \times (N+M)}.
\end{equation}
Similarly, we combine the orthogonal transforms into a single matrix,
\begin{equation}
\label{O}
O = [O_1 \ \cdots \ O_M] \in \mathbb{R}^{d \times Md}.
\end{equation}
Recall that we will conveniently identify $O$ with an element of $\mathbb{O}(d)^M$.

To  express \eqref{obj} in terms of $Z$ and $O$, we write $x_k - t_i = Ze_{ki}$,
where
\begin{equation*}
e_{ki} = e^{N+M}_k-e^{N+M}_{N+i}.
\end{equation*}
Similarly, we write $O_i = O (e^{M}_i \otimes I_d)$. This gives us
\begin{equation*}
\phi(Z,O) =   \displaystyle \sum_{(k,i) \in E(\Gamma)} \ \lVert Ze_{ki} - O (e^{M}_i \otimes I_d) x_{k,i}  \rVert^2.
\end{equation*}

Using $\lVert x \rVert^2=\Tr(xx^T)$, and properties of the trace, we obtain
\begin{equation}
\label{phiZO}
\phi(Z,O) = \Tr \left([ Z \ \ O] \begin{bmatrix} L& -B^T \\ -B & D \end{bmatrix} \begin{bmatrix} Z^T\\O^T \end{bmatrix} \right),
\end{equation}
where
\begin{eqnarray}
\label{coeff}
L &=&  \sum_{(k,i) \in E} \ e_{ki}e_{ki}^T,  \quad B =  \B, \ \text{and} \\ \nonumber
D &=& \D.
\end{eqnarray}
The matrix $L$ is  the combinatorial graph Laplacian of $\Gamma$ \cite{chung1997spectral}, and is of size $(N+M) \times (N+M)$. The matrix $B$ is of size $Md \times (N+M)$, and
the size of the  block diagonal matrix $D$ is $Md \times Md$.

The optimization program now reads
\begin{equation*}
(\p) \qquad \min_{Z,O}\ \phi(Z,O) \quad \text{subject to} \quad  Z \in \mathbb{R}^{d \times (N+M)}, \ O \in \mathbb{O}(d)^M.
\end{equation*}
The fact that $\mathbb{O}(d)^M$ is non-convex makes $(\p)$ non-convex.
In the next few subsections, we will show how this non-convex program can be approximated by tractable spectral and convex programs.

\subsection{Optimization over translations} Note that we can write $(\p)$ as
\begin{equation*}
 \min_{O \in \mathbb{O}(d)^M} \ \Big[\min_{Z \in \mathbb{R}^{d \times (N+M)}} \ \phi(Z,O) \Big].
\end{equation*}
That is, we first minimize over the free variable $Z$ for some fixed $O \in \mathbb{O}(d)^M$, and then we minimize with respect to $O$.

Fix some arbitrary $O \in \mathbb{O}(d)^M$, and set $\psi(Z)=\phi(Z,O)$. It is clear from \eqref{phiZO} that $\psi(Z)$ is quadratic in $Z$. In particular, the stationary points
$\Zs= \Zs(O)$ of $\psi(Z)$ satisfy
\begin{equation}
\label{minZ}
\nabla \psi(\Zs) = 0 \quad \Rightarrow \quad \Zs L=OB.
\end{equation}
The Hessian of $\psi(Z)$ equals $2L$, and it is clear from \eqref{coeff} that $L \succeq 0$. Therefore, $\Zs$ is a minimizer of $\psi(Z)$. 

If $\G$ is connected, then $e$ is the only vector in the null space of $L$ \cite{chung1997spectral}.
Let $L^{\dagger}$ be the Moore-Penrose pseudo-inverse of $L$, which is again positive semidefinite. It can be verified that
\begin{equation}
\label{L_Linv}
 L L^{\dagger} = L^{\dagger} L = I_{N+M} - (N+M)^{-1} ee^T.
\end{equation}
If we right multiply \eqref{minZ} by $L^{\dagger}$, we get
\begin{equation}
\label{inversion}
\Zs = OBL^{\dagger} + te^T,
\end{equation}
where $t \in \mathbb{R}^d$ is some global translation. Conversely, if we right multiply \eqref{inversion} by $L$ and use the facts that $e^T L=0$ and $Be=0$, we get \eqref{minZ}.
Thus, every solution of \eqref{minZ} is of the form \eqref{inversion}.

Substituting \eqref{inversion} into \eqref{phiZO}, we get
\begin{equation}
\label{penalty3}
\psi(\Zs) = \phi(\Zs,O)= \Tr(C O^TO) = \sum_{i,j=1}^M \Tr(O_i C_{ij} O_j^T),
\end{equation}
where
\begin{equation}
\label{stressC}
C = \begin{bmatrix} BL^{\dagger}&I_{Md} \end{bmatrix} \begin{bmatrix} L & -B^T \\-B & D \end{bmatrix} \begin{bmatrix} L^{\dagger}B^T \\ I_{Md} \end{bmatrix} = D - B L^{\dagger} B^T.
\end{equation}
Note that \eqref{penalty3} has the global translation $t$ taken out. This is not a surprise since $\phi$ is invariant to global translations.
Moreover, note that we have not forced the orthogonal constraints on $O$ as yet. Since $\phi(Z,O) \geq 0$ for any $Z$ and $O$, it necessarily follows from \eqref{penalty3} that $C \succeq 0$.
We will see in the sequel how the spectrum of $C$ dictates the performance of the convex relaxation of \eqref{penalty3}.

In analogy with the notion of stress in rigidity theory \cite{gortler2010characterizing}, we can consider \eqref{obj} as a sum of the ``stress'' between pairs of patches when we
try to register them using rigid transforms. In particular, the $(i,j)$-th term in \eqref{penalty3} can be regarded as the stress between the (centered) $i$-th and $j$-th patches generated by the
orthogonal transforms.  Keeping this analogy in mind, we will henceforth refer to $C$ as the \textit{patch-stress matrix}.

\subsection{Optimization over orthogonal transforms}

The goal now is to optimize \eqref{penalty3} with respect to the orthogonal transforms, that is, we have reduced $(\p)$ to the following problem:
\begin{equation*}
(\p_0) \qquad \min_{O \in \mathbb{R}^{d \times Md}} \ \Tr(C O^TO) \quad \text{subject to} \quad (O^TO)_{ii} = I_d \ (1 \leq i \leq M).
\end{equation*}
This is a non-convex problem since $O$ lives on a non-convex (disconnected) manifold \cite{absil2009optimization}.
We will generally refer to any method which uses manifold optimization to solve $(\p_0)$ and
then computes the coordinates using \eqref{inversion} as ``Global Registration over Euclidean Transforms using Manifold Optimization'' (\texttt{GRET-MANOPT}).

\subsection{Spectral relaxation and rounding}

Following the quadratic nature of the objective in $(\p_0)$, it is possible to relax it into a spectral problem. More precisely, consider the domain
\begin{equation*}
\mathcal{S} = \{ O \in \mathbb{R}^{d \times Md} \ : \ \text{rows of $O$ are orthogonal and each row has norm $\sqrt{M}$}\}.
\end{equation*}
That is, we do not require the $d \times d$ blocks in $O \in \mathcal{S}$ to be orthogonal. Instead, we only require the rows of $O$ to form an orthogonal system, and each row to have the same norm.
It is clear that $\mathcal{S}$ is a larger domain than that determined by the constraints in $(\p_0)$. In particular, we consider the following relaxation of $(\p_0)$:
\begin{equation*}
(\p_1) \qquad \min_{O \in \mathcal{S}} \Tr(C O^TO).
\end{equation*}
This is precisely a spectral problem in that the global minimizers are determined from the spectral decomposition of $C$. More precisely, let $\mu_1 \leq \ldots \leq\mu_{Md}$ be eigenvalues of $C$,
and let  $r_1,\ldots,r_{Md}$ be the corresponding eigenvectors. Define
\begin{equation}
\label{O_spec}
W^{\star}  \defn \sqrt{M} \big[r_1 \cdots r_d  \big]^T \in  \mathbb{R}^{d \times Md} .
\end{equation}
Then
\begin{equation}
\label{obj_SPEC}
\Tr(C {W^{\star}}^T W^{\star} ) = \min_{O \in \mathcal{S}} \Tr(C O^TO) = M(\mu_1+\cdots+\mu_d).
\end{equation}

Due to the relaxation, the blocks of $W^{\star} $ are not guaranteed to be in $\mathbb{O}(d)$. We round each $d \times d$ block of $W^{\star} $ to its ``closest'' orthogonal matrix. More precisely, let $W^{\star}  = [W^{\star} _1 \cdots W^{\star} _M]$.
For every $1 \leq i \leq M$, we find $O_i^{\star} \in \mathbb{O}(d)$ such that
\begin{equation*}
 \lVert   \Os_i - W^{\star} _i \rVert_F = \min_{O \in \mathbb{O}(d)}  \  \lVert  O- W^{\star} _i \rVert_F.
\end{equation*}
As noted earlier, this has a closed-form solution, namely $O_i^{\star}=UV^T,$ where $U \Sigma V^T$ is the SVD of $W^{\star} _i$.
We now put the rounded blocks back into place and define
\begin{equation}
\label{optO}
\Os \defn \big[ \Os_1 \ldots \Os_M \big] \in \mathbb{O}(d)^M.
\end{equation}
In the final step, following \eqref{inversion}, we define
\begin{equation}
\label{optZ}
\Zs \defn \Os B L^{\dagger} \in \mathbb{R}^{d \times (N+M)}.
\end{equation}
The first $N$ columns of $\Zs$ are taken to be the reconstructed global coordinates.

We will refer to this spectral method as the ``Global Registration over Euclidean Transforms using Spectral Relaxation'' (\texttt{GRET-SPEC}). The main steps of \texttt{GRET-SPEC}
are summarized in Algorithm \ref{algo:SPEC}. We note that a similar spectral algorithm was proposed
for angular synchronization by Bandeira et al. \cite{bandeira2012cheeger}, and by Krishnan et al. \cite{GlobalRegKrishnan} for initializing the manifold optimization.

\begin{algorithm}
\caption{\texttt{GRET-SPEC}}
\begin{algorithmic}[1]
\REQUIRE Membership graph $\G$, local coordinates $\{x_{k,i}, (k,i) \in E(\G)\}$, dimension $d$.
\ENSURE Global coordinates $x_1,\ldots,x_N$ in $\mathbb{R}^d$.
\STATE  Build $B,L$ and $D$ in \eqref{coeff} using $\G$.
\STATE  Compute $L^{\dagger}$ and $C= D - B L^{\dagger} B^T $.
\STATE  Compute bottom $d$ eigenvectors of $C$, and set $W^{\star} $ as in \eqref{O_spec}.
\FOR {$i = 1$ to $M$}
\STATE $\Os_i \gets U_iV_i^T$.
\ENDFOR
\STATE $\Os \gets  \big[\Os_1  \cdots  \Os_M \big]$
\STATE $\Zs \gets  \Os B L^{\dagger}$.
\STATE Return first $N$ columns of $\Zs$.
\end{algorithmic}
\label{algo:SPEC}
\end{algorithm}

The question at this point is how are the quantities $\Os$ and $\Zs$ obtained from  \texttt{GRET-SPEC} related to the original problem $(\p)$?
Since $(\p_1)$ is obtained by relaxing the block-orthogonality constraint in $(\p_0)$, it is clear  that if the blocks of $W^{\star}$ are orthogonal, then $\Os$ and $\Zs$ are solutions of $(\p)$, that is,
\begin{equation*}
\phi(\Zs, \Os) \leq  \phi(Z,O) \qquad  \text{for all } Z \in \mathbb{R}^{d \times (N+M)}, \ O \in \mathbb{O}(d)^M.
\end{equation*}
We have actually found the global minimizer of the original non-convex problem $(\p)$ in this case.
\begin{observation}[Tight relaxation using \texttt{GRET-SPEC}]
\label{obs:tight_SPEC}
If the $d \times d$ blocks of the solution of $(\p_1)$ are orthogonal, then the coordinates and transforms computed by \texttt{GRET-SPEC} are the global minimizers of $(\p)$.
\end{observation}

If some the blocks are not orthogonal, the rounded quantities $\Os$ and $\Zs$ are only an approximation of the solution of $(\p)$.

\subsection{Semidefinite relaxation and rounding}

We now explain how we can obtain a tighter relaxation of $(\p_0)$ using a semidefinite program, for which the global
minimizer can be computed efficiently. Our semidefinite program was motivated by the line of works on the semidefinite relaxation
of non-convex problems \cite{lovasz1991cones,goemans1995improved,vandenberghe1996semidefinite,candes2009exact}.

Consider the domain
\begin{equation*}
\mathcal{C} = \{ O \in \mathbb{R}^{Md \times Md} \ : \  (O^TO)_{11}=\cdots=(O^TO)_{MM}=I_d\}.
\end{equation*}
That is, while we require the columns of each $Md \times d$ block of $O \in \mathcal{C}$ to be orthogonal, we do not force the non-convex rank constraint $\mathrm{rank}(O)=d$.
This gives us the following relaxation
\begin{equation}
\label{p2}
\min_{O \in \mathcal{C}} \Tr(C O^TO).
\end{equation}
Introducing the variable $G = O^TO$, \eqref{p2} is equivalent to
\begin{equation*}
(\p_2) \qquad \min_{G \in \mathbb{R}^{Md \times Md}} \Tr(CG) \quad \text{ subject to} \quad G \succeq 0, \ G_{ii} =  I_d \ (1 \leq i \leq M).
 \end{equation*}
This is a standard semidefinite program \cite{vandenberghe1996semidefinite} which can be solved using software packages such as \texttt{SDPT3} \cite{toh1999sdpt3} and \texttt{CVX} \cite{grant2008cvx}.
We provide details about SDP solvers and their computational complexity later in Section \ref{comp}.

Let us denote the solution of $(\p_2)$ by $\Gs$, that is,
\begin{equation}
\label{optG}
 \Tr(C\Gs) = \min_{G \in \mathbb{R}^{Md \times Md}} \  \{ \Tr(CG) : G \succeq 0, \ G_{11} = \cdots = G_{MM} = I_d  \}.
\end{equation}
By the linear constraints in $(\p_2)$, it follows that $\rnk(\Gs) \geq d$. If $\rnk(\Gs)>d$, we need to round (approximate) it by a rank-$d$ matrix.
That is, we need to project it onto
the domain of $(\p_0)$. One possibility would be to use random rounding that come with approximation guarantees; for example, see \cite{goemans1995improved,Ortho-Cut}.
In this work, we use deterministic rounding, namely the eigenvector rounding which retains the top $d$ eigenvalues and discards the remaining.
In particular, let $\lambda_1 \geq \lambda_2 \geq \cdots  \geq \lambda_{Md}$ be the eigenvalues of $\Gs$, and $q_1,\ldots,q_{Md}$ be the corresponding eigenvectors. Let
\begin{equation}
\label{rank_enforce}
W^{\star} \defn \big[ \sqrt{\lambda_1} q_1 \  \cdots \ \sqrt{\lambda_d}  q_d \big]^T \in \mathbb{R}^{d \times Md}.
\end{equation}

We now proceed as in the \text{GRET-SPEC}, namely, we define $\Os$ and $\Zs$ from $W^{\star}$ as in \eqref{optO} and \eqref{optZ}.
We refer to the complete algorithm as ``Global Registration over Euclidean Transforms using SDP'' (\texttt{GRET-SDP}).
The main steps of \texttt{GRET-SDP} are summarized in Algorithm \ref{algo:GRET}.

\begin{algorithm}
\caption{\texttt{GRET-SDP}}
\begin{algorithmic}[1]
\REQUIRE Membership graph $\G$, local coordinates $\{x_{k,i}, (k,i) \in E(\G)\}$, dimension $d$.
\ENSURE  Global coordinates $x_1,\ldots,x_N$ in $\mathbb{R}^d$.
\STATE   Build $B,L$ and $D$ in \eqref{coeff} using $\G$.
\STATE  Compute $L^{\dagger}$ and $C= D - B L^{\dagger} B^T $.
\STATE  $\Gs \gets$ Solve the SDP $(\p_2)$ using $C$.
\STATE Compute top $d$ eigenvectors of $\Gs$, and set $W^{\star}$ using \eqref{rank_enforce}.
\FOR {$i = 1$ to $M$}
\STATE Compute $W_i^{\star}= U_i \Sigma_i V_i^T$.
\STATE $\Os_i \gets U_iV_i^T$.
\ENDFOR
\STATE $\Os \gets  \big[\Os_1  \cdots  \Os_M \big]$
\STATE $\Zs \gets \Os B L^{\dagger}$.
\STATE Return first $N$ columns of $\Zs$.
\end{algorithmic}
\label{algo:GRET}
\end{algorithm}

Similar to Observation \ref{obs:tight_SPEC}, we note the following for \texttt{GRET-SDP}.
\begin{observation}[Tight relaxation using \texttt{GRET-SDP}]
\label{obs:tight_SDP}
If the rank of the solution of $(\p_2)$ is exactly $d$, then the coordinates and transforms computed by \texttt{GRET-SDP} are the global minimizers of $(\p)$.
\end{observation}

If $\rnk(G^{\star})>d$, the output of \texttt{GRET-SDP} can only be considered as an approximation of the solution of $(\p)$.
The quality of the approximation for $(\p_2)$ can be quantified using, for example, the randomized rounding in \cite{Ortho-Cut}. More precisely, note that since $D$ is block-diagonal,
\eqref{p2} is equivalent (up to a constant term) to
\begin{equation*}
\max_{O \in \mathcal{C}} \ \Tr(Q O^T O)
\end{equation*}
where $Q = BL^{\dagger}B^T  \succeq 0$. Bandeira et al. \cite{Ortho-Cut} show that the orthogonal transforms (which we continue to denote by $\Os$) obtained by a certain random rounding of $\Gs$ satisfy
\begin{equation*}
 \mathbb{E} \left[ \Tr(Q  \ {\Os}^T \Os) \right] \geq \alpha_d^2 \cdot \mathrm{OPT},
\end{equation*}
where $\mathrm{OPT}$ is the optimum of the unrelaxed problem \eqref{MC} with $Q=BL^{\dagger}B^T $, and $\alpha_d$ is the expected average of the singular values of a $d \times d$ random matrix with
entries iid $\mathcal{N}(0,1/d)$. It was conjectured in \cite{Ortho-Cut} that $\alpha_d$ is monotonically increasing, and the boundary values were computed
to be $\alpha_1=\sqrt{2/\pi}$ ($\alpha_1$ was also reported here \cite{nesterov1998semidefinite}) and $\alpha_{\infty}=8/3\pi$.  We refer the reader to \cite{Ortho-Cut} for further details on the rounding procedure, and its relation to previous work in terms of the approximation ratio. Empirical results, however, suggest that the difference between deterministic and randomized rounding is
small as far as the final reconstruction is concerned. We will therefore simply use the deterministic rounding.

\subsection{Computational complexity}
\label{comp}

The main computations in \texttt{GRET-SPEC} are the Laplacian inversion, the eigenvector computation, and the orthogonal rounding.
The cost of inverting $L$ when $\Gamma$ is dense is $O((N+M)^3)$. However, for most practical applications, we expect $\Gamma$ to be sparse since every point would typically
be contained in a small number of patches. In this case, it is known that the linear system $Lx=b$ can be solved in time almost linear in the number of edges in $\Gamma$ \cite{spielman2004nearly,Vishnoi2012}.
Applied to \eqref{L_Linv}, this means that
we can compute $L^{\dagger}$
in $O(|E(\Gamma)|)$ time (up to logarithmic factors). Note that, even if $L$ is dense, it is still possible to speed up the inversion (say, compared to a direct Gaussian elimination)
using the formula \cite{ho2005pseudo,ranjan2013incremental}:
\begin{equation*}
 L^{\dagger} = [L + (N+M)^{-1} ee^T]^{-1} -  (N+M)^{-1} ee^T.
\end{equation*}
The speed up in this case is however in terms of the absolute run time. The overall complexity is still $O((N+M)^3)$, but with smaller constants. We note that it is also possible to speed up the inversion by
exploiting the bipartite nature of $\Gamma$ \cite{ho2005pseudo}, although we have not used this in our implementation.

The complexity of the eigenvector computation is $O(M^3d^3)$, while that of the orthogonal rounding is $O(Md^3)$. The total complexity of  \texttt{GRET-SPEC}, say, using a linear-time
Laplacian inversion, is (up to logarithmic factors)
\begin{equation*}
O\left(|E(\Gamma)|+ (Md)^3 \right).
\end{equation*}

The main computational blocks in \texttt{GRET-SDP} are identical to that in \texttt{GRET-SPEC}, plus the SDP computation. The SDP solution can be computed in polynomial time using interior-point programming \cite{SDP_handbook}.
In particular, the complexity of computing an $\varepsilon$-accurate solution using interior-point solvers
such as \texttt{SDPT3} \cite{toh1999sdpt3} is $O((Md)^{4.5} \log(1/\varepsilon))$. It is possible to lower this complexity by exploiting the particular structure of $(\p_2)$. For example, notice that the constraint
matrices in $(\p_2)$ have at most one non-zero coefficient. Using the algorithm in \cite{helmberg1996interior}, one can then bring down the complexity of the SDP to $O((Md)^{3.5} \log(1/\varepsilon))$.
By considering a penalized version of the SDP, we can use first-order solvers such as \texttt{TFOCS} \cite{becker2011templates} to further cut
down the dependence on $M$ and $d$ to $O((Md)^3 \varepsilon^{-1})$, but at the cost of a stronger dependence on the accuracy.  The quest for efficient SDP solvers is currently an active area of research.
Fast SDP solvers have been proposed that exploit either the low-rank structure of the SDP solution \cite{burer2003nonlinear,journee2010low} or the simple form of the linearity constraints in $(\p_2)$  \cite{wen2012block}.
More recently, a sublinear time approximation algorithm for SDP was proposed in \cite{garber2012almost}. The complexity of \texttt{GRET-SDP} using a linear-time
Laplacian inversion and an interior-point SDP solver is thus
\begin{equation*}
O\left(|E(\Gamma)|+ (Md)^{4.5} \log(1/\varepsilon) + (Md)^3 \right).
\end{equation*}
For problems where the size of the SDP variable is within $150$, we can solve $(\p_2)$ in reasonable time on a standard PC using \texttt{SDPT3} \cite{toh1999sdpt3} or \texttt{CVX} \cite{grant2008cvx}.
We use \texttt{CVX} for the numerical experiments in Section \ref{Sim} that involve small-to-moderate sized SDP variables. For larger SDP variables, one can use the low-rank structure of $(\p_2)$ to speed up the computation. In particular, we were able to solve for SDP variables of size up to $2000 \times 2000$ using \texttt{SDPLR} \cite{burer2003nonlinear} that exploits this low-rank structure .

\section{Exact Recovery}
\label{ER}

We now examine conditions on the membership graph under which the proposed spectral and convex relaxations can recover the global coordinates from the knowledge of the clean local coordinates (and the membership graph). More precisely, let $\bar x_1,\ldots, \bar x_N$ be the true coordinates of a point cloud in $\mathbb{R}^d$. Suppose that the point cloud is divided into patches whose membership graph is $\G$, and
that we are provided the measurements
\begin{equation}
\label{clean_data}
x_{k,i} = \bar{O}_i^T (\bar{x}_k - \bar{t}_i) \qquad (k,i) \in E(\G),
\end{equation}
for some $\bar{O}_i \in \mathbb{O}(d)$ and $\bar{t}_i \in \mathbb{R}^d$.  The patch-stress matrix $C$ is constructed from $\Gamma$ and the clean measurements \eqref{clean_data}.
The question is under what conditions on $\Gamma$ can $\bar {x}_1,\ldots,\bar x_N$ be  recovered by our algorithm?  We will refer to this as \textit{exact recovery}.

To express \textit{exact recovery} in the matrix notation introduced earlier, define
\begin{equation*}
\bar{Z} = \big[\bar{x}_1 \ \cdots \ \bar{x}_N \ \bar{t}_1 \ \cdots \ \bar{t}_M \big] \in \mathbb{R}^{d \times (N+M)},
\end{equation*}
and
\begin{equation*}
\bar{O} = [ \bar{O}_1 \ \cdots \ \bar{O}_M] \in \mathbb{R}^{d \times Md}.
\end{equation*}
Then, exact recovery means that  for some $\Omega \in \mathbb{O}(d)$ and $t \in \mathbb{R}^d$,
\begin{equation}
\label{exact_rec}
\Zs = \Omega \bar{Z} + te^T.
\end{equation}
Henceforth, we will always assume that $\Gamma$ is connected (clearly one cannot have exact recovery otherwise).

Conditions for exact recovery have previously been examined by Zha and Zhang \cite{zha2009spectral} in the context of tangent-space alignment in manifold learning, and later by Gortler et al. \cite{gortler2013affine} from the perspective of rigidity theory. In particular, they show that the so-called notion of \textit{affine rigidity} is sufficient for exact recovery using the spectral method. Moreover, the authors in \cite{zha2009spectral,gortler2013affine} relate this notion of rigidity to other standard notions of rigidity, and provide conditions on a certain hypergraph constructed from the patch system that can guarantee affine rigidity.  The purpose of this section is to briefly introduce the rigidity results in \cite{zha2009spectral,gortler2013affine} and relate these to the properties of the membership graph $\G$ (and the patch-stress matrix $C$). We note that the authors in \cite{zha2009spectral,gortler2013affine} directly examine the uniqueness of the global coordinates, while we are concerned with the uniqueness of the
patch transforms obtained by solving $(\p_1)$ and $(\p_2)$. The uniqueness of the global coordinates is then immediate:
\begin{proposition}[Uniqueness and Exact Recovery]
\label{uniqueness_exact}
If $(\p_1)$ and $(\p_2)$ have unique solutions, then \eqref{exact_rec} holds for both \texttt{GRET-SPEC} and \texttt{GRET-SDP}.
\end{proposition}

At this point, we  note that if a patch has less than $d+1$ points, then even when $\bar x_1,\ldots, \bar x_N$ are the unique set of coordinates that satisfy \ref{clean_data}, we cannot guarantee  $\bar O_1,\ldots, \bar O_M$ and $\bar t_1,\ldots,\bar t_M$ to be unique. Therefore, we will work under the mild assumption that each patch has at least $d+1$ non-degenerate points, so that the patch transforms are uniquely determined from the global coordinates.

We now formally define the notion of affine rigidity. Although phrased differently, it is in fact identical to the definitions in \cite{zha2009spectral,gortler2013affine}. Henceforth, by affine transform, we will mean the group of non-singular affine maps on
$\mathbb{R}^d$. Affine rigidity is a property of the patch-graph $\G$ and the local coordinates $(x_{k,i})$.
In keeping with \cite{gortler2013affine}, we will together call these the \textit{patch framework} and denote it by $\Theta=(\G,(x_{k,i}))$.

\begin{definition}[Affine Rigidity]
Let $y_1,\ldots,y_N \in \mathbb{R}^d$ be such that, for affine transforms $A_1,\ldots,A_M$,
\begin{equation*}
y_k = A_i (x_{k,i}) \qquad (k,i) \in E(\G).
\end{equation*}
The patch framework $\Theta=(\G,(x_{k,i}))$ is affinely rigid if $y_1, \ldots, y_N$ is identical to $\bar x_1,\ldots \bar x_N$ up to a global affine transform.
\end{definition}

Since each patch has $d+1$ points, we now give a characterization of affine rigidity that will be useful later on.
\begin{proposition}
\label{iffaffine}
 A patch framework $\Theta=(\G,(x_{k,i}))$ is affinely rigid if and only if for any $F\in R^{d \times Md}$ such that $\Tr(C F^T F) = 0$ we must have  $F = A \bar O$ for some non-singular $A \in R^{d \times d}$.
\end{proposition}

Before proceeding to the proof, note that  $\bar{O}$ and $\bar{G}=\bar{O}^T\bar{O}$ are solutions of $(\p_1)$  and $(\p_2)$ (this was the basis of Proposition \ref{uniqueness_exact}), and the objective
 in either case is zero. Indeed, from \eqref{clean_data}, we can write $\bar Z L = \bar O B$. Since $\G$ is connected,
\begin{equation}
\label{stationary}
\bar Z = \bar{O} BL^{\dagger} + te^T \qquad (t \in \mathbb{R}^d).
\end{equation}
Using \eqref{stationary}, it is not difficult to verify that $\phi(\bar{Z}, \bar{O}) = \Tr(C \bar{G})$. Moreover, it follows from \eqref{clean_data} that $\phi(\bar{Z}, \bar{O}) = 0$.
Therefore,
\begin{equation}
\label{zeroCost}
\Tr(C \bar{G}) = \Tr(C {\bar{O}}^T \bar{O}) = 0.
\end{equation}

Using an identical line of reasoning, we also record another fact. Let $F = [F_1,\ldots,F_M]$ where each $F_i \in \mathbb{R}^{d \times d}$. Suppose there exists $y_1,\ldots,y_N \in \mathbb{R}^d$ and $t_1,\ldots,t_M \in \mathbb{R}^d$ such that
\begin{equation}
\label{affinecon}
y_k = F_i x_{k,i} + t_i \qquad (k,i) \in E(\G).
\end{equation}
Then $Y = [y_1,\ldots,y_N,t_1,\ldots,t_M] \in \mathbb{R}^{d\times (N+M)}$ satisfies
\begin{equation}
\label{affinesol}
Y = FBL^\dagger + te^T
\end{equation}
and $\Tr(CF^T F)=0$.

\begin{proof} [of Proposition \ref{iffaffine}]
For any $F$ such that $\Tr( C F^T F) = 0$, letling $$[y_1,\ldots,y_N, t_1,\ldots, t_M] = F B L^\dagger,$$ we have  \eqref{affinecon}. By the affine rigidity assumption, we must then have $y_k = A \bar x_k + t$ for some non-singular $A\in \mathbb{R}^{d\times d}$ and $t\in \mathbb{R}^d$. Since each patch contains $d+1$ non-degenerate points, it follows that $F = A \bar O$.

In the other direction, assume that  $y_1,\ldots,y_N \in \mathbb{R}^d$ satisfy \eqref{affinecon}. We know that $\Tr(C F^T F) = 0$ and hence $F = A \bar O$ for some non-singular $A$. Using \eqref{affinesol}, we immediately have $y_k = A\bar{x}_k + t$.
\end{proof}

Note that $\Tr(C F^T F) = 0$ implies that the rows of $F$ are in the null space of $C$. Therefore, the combined facts that  $\Tr(C F^T F) = 0$ and  $F = A \bar O$ for some non-singular $A \in R^{d \times d}$ is equivalent to saying that null space of $C$ is within the row span of $\bar O$.  The following result then follows as a consequence of \eqref{iffaffine}.
\begin{corollary}
\label{rankstress}
A patch framework $\Theta=(\G,(x_{k,i}))$ is affinely rigid if and only if the rank of $C$ is $(M-1)d$.
\end{corollary}


\begin{figure}
\centering
\includegraphics[width=0.7\linewidth]{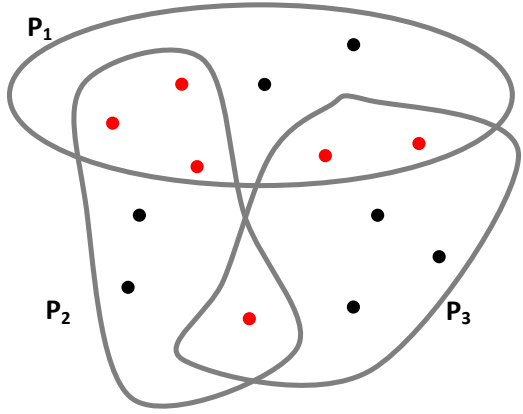}
\caption{Instance of three overlapping patches, where the overlapping points are shown in red. In this case, $P_3$ cannot be registered with either $P_1$ or $P_2$ due to insufficient overlap.
Therefore, the patches cannot be localized in two dimension using, for example, \cite{zhang2010arap,cucuringu2013asap} that work by registering pairs of patches.
The patches can however be registered using \texttt{GRET-SPEC} and \texttt{GRET-SDP} since the ordered patches $P_1,P_2,P_3$ form a
graph lateration in $\mathbb{R}^2$.}
\label{laterated_patches}
\end{figure}

The corollary gives an easy way to check for affine rigidity. However, it is not clear what construction of $\G$ will ensure such property. In \cite{zha2009spectral}, the notion of graph \textit{lateration} was introduced that guarantees affine rigidity. Namely, $\G$ is said to be a graph lateration (or simply laterated) if there exists an reordering of the patch indices such that, for every $i \geq 2$,  $P_i$ and $P_1 \cup \cdots \cup P_{i-1}$ have at least $d+1$ non-degenerate nodes in common. An example of a graph lateration is shown in Figure \ref{laterated_patches}.

\begin{theorem}[\cite{zha2009spectral}]
\label{thm:LaterationRecovery}
If $\G$ is laterated and the local coordinates are non-degenerate then the framework $\Theta$ is affinely rigid.
\end{theorem}

Next, we turn to  the exact recovery conditions for $(\p_2)$.
The appropriate notion of rigidity in this case is that of universal rigidity \cite{gortler2009characterizing}.
Just as we defined affine rigidity earlier, we can phrase universal rigidity as follows.

\begin{definition}[Universal Rigidity]
Suppose that \eqref{clean_data} holds. Let $x_1,\ldots,x_N \in \mathbb{R}^s (s \geq d)$ be such that, for some orthogonal $O_i \in \mathbb{R}^{s \times d}$ and $t_i \in \mathbb{R}^s$,
$$x_k = O_i x_{k,i} + t_i \qquad (k,i) \in E.$$
We say that the patch framework $\Theta=(\G,(x_{k,i}))$ is universally rigid in $\mathbb{R}^d$ if for any such $(x_k)$,  we have $ x_k = \Omega \bar x_k$ for some orthogonal $\Omega \in \mathbb{R}^{s \times d}$.
\end{definition}

By orthogonal $\Omega$, we mean that the columns of $\Omega$ are orthogonal and of unit norm (i.e., $\Omega$ can be seen an orthogonal transform in $\mathbb{R}^s$ by identifying $\mathbb{R}^d$ as a subspace of $\mathbb{R}^s$).

Following exactly the same arguments used to establish Proposition \ref{iffaffine}, one can derive the following.
\begin{proposition} The following statements are equivalent: \\
$\mathrm{(a)}$  A patch framework $\Theta=(\G,(x_{k,i}))$ is universally rigid in $\mathbb{R}^d$. \\
$\mathrm{(b)}$  Let $O\in \mathbb{R}^{s \times Md} (s\geq d)$ be such that $O_i^T O_i = I_d$ for all $i$. Then 
\begin{equation*}
\Tr(C O^T O) = 0 \quad \Rightarrow \quad O = \Omega \bar O \text{ for some orthogonal } \Omega \in \mathbb{R}^{s \times d}.
\end{equation*}
\end{proposition}
The question then is under what conditions is the patch framework universally rigid? This was also addressed in \cite{gortler2013affine} using a graph construction derived from $\Gamma$ called the \textit{body graph}. This is given by $\G_B=(V_B, E_B)$, where $V_B = \{1,2,\ldots,N \}$ and $(k,l) \in E_B$ if and only if  $x_k$ and $x_l$ belong to the same patch (cf. Figure \ref{BBG}).
Next, the following distances are associated with $\G_B$:
\begin{equation}
\label{distances}
d_{kl}=  \Vert x_{k,i} - x_{l,i} \rVert  \qquad (k,l) \in E_B,
\end{equation}
where $x_k, x_l \in P_i$, say. Note that the above assignment is independent of the choice of patch. A set of points $(x_k)_{k \in V}$ in $\mathbb{R}^s$ is said to be a \textit{realization} of $\{d_{kl} : (k,l) \in E \}$ in $ \mathbb{R}^s$ if $d_{kl} = ||x_k - x_l||$ for $(k,l) \in E$.

\begin{figure}
\centering
\includegraphics[width=0.7\linewidth]{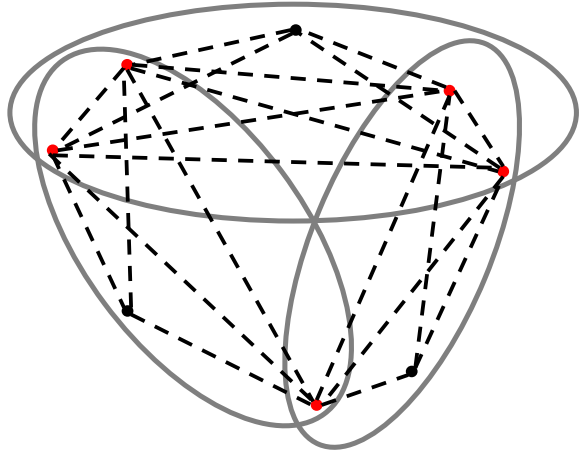}
\caption{This shows the body graph for a $3$-patch system (patches marked with ovals, points marked with dots). The edges of the body graph are obtained by connecting points that belong to the same patch.
The edges within a given patch are marked with the same color. \texttt{GRET-SDP} can successfully register all the patches if the body graph is rigid in a certain sense.}
\label{BBG}
\end{figure}

It was shown in \cite{gortler2013affine} that  $\Theta=(\G,(x_{k,i}))$  is universally rigid if and only if $\G_B$ with distances $\{d_{kl} : (k,l) \in E \}$ has a unique realization in $\mathbb{R}^s$ for all $s \geq d$. Moreover, in such situation, using the distances as the constraints, an SDP relaxation was proposed in \cite{so2007theory,zhu2010universal} for finding the unique realization. We note that although the SDP in \cite{so2007theory} has the same condition for exact recovery as $(\p_2)$, it is computationally more demanding than $(\p_2)$ since the number of variables is  $O(N^2)$ for this SDP, instead of $O(M^2)$ as in $(\p_2)$ (for most applications, $M \ll N$). Moreover, as we will see shortly in Section \ref{Sim}, $(\p_2)$ also enjoys some stability properties which has not been established for the SDP in \cite{so2007theory}.

Finally, we note that universal rigidity is a weaker condition on $\G$ than affine rigidity.

\begin{theorem}[\cite{so2007theory}, Theorem 2]
If a patch framework is affinely rigid, then it is universally rigid.
\end{theorem}

In \cite{gortler2013affine}, it was also shown that the reverse implication is not true using an counter-example for which the patch framework fails to be affinely rigid, but for which the body graph (a Cauchy polygon) has an unique realization in any dimension \cite{connelly1982rigidity}. This means that $\texttt{GRET-SDP}$ can solve a bigger class of problems than $\texttt{GRET-SPEC}$, which is perhaps not surprising.

\section{Randomized Rank Test}
\label{randTest}

Corollary \ref{rankstress} tells us by checking the rank of the patch stress matrix $C$, we can tell whether a patch framework is affinely rigid. In this regard, the patch-stress matrix serves the same purpose as the so-called alignment matrix in \cite{zha2009spectral} and the affinity matrix in \cite{gortler2013affine}. The only difference is that the kernel of $C$ represents the degree of freedom of the affine transform, whereas kernel of alignment or affinity matrix directly tell us the degree of freedom of the point coordinates. As suggested in \cite{gortler2013affine}, an efficient randomized test for affine rigidity using the concept of affinity matrix can be easily derived. In this section, we describe a randomized test based on patch stress matrix, which parallels the proposal in \cite{gortler2013affine}. This procedure is also similar in spirit to the randomized tests for generic local rigidity by Hendrickson \cite{hendrickson1992conditions}, for generic global rigidity by Gortler et al. \cite{gortler2010characterizing}, and for matrix
completion by Singer and Cucuringu \cite{MatCompletionSinger}.

Let us continue to denote the patch-stress matrix obtained from $\G$ and the measurements \eqref{clean_data} by $C$. We will use $C_0$ to denote the patch-stress matrix obtained
from the same graph $\G$, but using the (unknown) original coordinates as measurements, namely,
\begin{equation}
\label{original_data}
 x_{k,i} = \bar x_k \qquad (k,i) \in \G.
\end{equation}
The advantage of working with $C_0$ over $C$ is that the former can be computed using just the global coordinates, while the latter requires the knowledge of the global coordinates as well as the clean transforms.
In particular, this only requires us to simulate the global coordinates. Since the coordinates of points in a given patch are determined up to a rigid transform, we claim the following (cf. Section \ref{proof:Rank_C0}
for a proof).
\begin{proposition}[Rank equivalence]
\label{prop:Rank_C0}
For a fixed $\Gamma$, $C$ and $C_0$ have the same rank.
\end{proposition}

In other words, the rank of $C_0$ can be used to certify exact recovery. The proposed test is based on Proposition \ref{proof:Rank_C0}, and the fact that if two different \textit{generic} configurations are used as
input in \eqref{original_data} (for the same $\G$), then the patch-stress matrices they produce would have the same rank. By generic,
we mean that the coordinates of the configuration do not satisfy any non-trivial algebraic equation with rational coefficients \cite{gortler2010characterizing}.

\begin{algorithm}[here]
\caption{\texttt{GRET-RRT}}
\begin{algorithmic}[1]
\REQUIRE  Membership graph $\G$, and dimension $d$.
\ENSURE   Exact recovery certificate for \texttt{GRET-SDP}.
\STATE Build $L$ using $\G$, and compute $L^{\dagger}$.
\STATE Randomly pick $\{x_1,\ldots,x_N\}$ from the unit cube in $\mathbb{R}^d$, where $N=|V_x(\Gamma)|$.
\STATE $x_{k,i} \gets x_k$ for every  $(k,i) \in E(\G)$.
\STATE $C_0 \gets D - BL^{\dagger}B^T$.
\IF {$\rnk(C_0)=(M-1)d$}
\STATE Positive certificate for \texttt{GRET-SPEC} and \texttt{GRET-SDP}.
\ELSE
\STATE  Negative certificate for \texttt{GRET-SPEC}.
\STATE  \texttt{GRET-SDP} cannot be certified.
\ENDIF
\end{algorithmic}
\label{algo:RRT}
\end{algorithm}	

The complete test called ``GRET-Randomized Rank Test'' (\texttt{GRET-RRT}) is described in Algorithm \ref{algo:RRT}.
Note that the main computations in \texttt{GRET-RRT} are the Laplacian inversion (which is also required for the registration algorithm) and the rank computation.
%

\section{Stability Analysis}
\label{Stab}

We have so far studied the problem of exact recovery from noiseless measurements. In practice, however, the measurements are invariably noisy.
This brings us to the question of stability, namely how stable are \texttt{GRET-SPEC} and \texttt{GRET-SDP} to perturbations in the measurements? Numerical results (to be presented in the next Section) show that both the spectral
and semidefinite relaxations are quite stable  to perturbations. In particular, the reconstruction error degrades quite gracefully with the increase in noise (reconstruction error is the gap between the outputs with clean and
noisy measurements). In this Section, we try to quantify these empirical observations. In particular, we prove that, for a specific noise model, the
reconstruction error grows at most linearly with the level of noise for the semidefinite relaxation.

The noise model we consider is the ``bounded'' noise model. Namely, we assume that the measurements are obtained through bounded perturbations of the clean measurements in \eqref{clean_data}.
More precisely, we suppose that we have a membership graph $\Gamma$, and that the observed local coordinates are of the form
\begin{equation}
\label{noisy_measurements'}
x_{k,i} =\bar{O}_i^T (\bar{x}_k - \bar{t}_i) + \epsilon_{k,i}, \qquad \|\epsilon_{k,i}\|\leq \varepsilon \quad (k,i) \in E(\G).
\end{equation}
In other words, every coordinate measurement is offset within a ball of radius $\varepsilon$ around the clean measurements. Here, $\varepsilon$ is a measure of the noise level per measurement.
In particular, $\varepsilon = 0$ corresponds to the case where we have the clean measurements \eqref{clean_data}.

Since the coordinates of points in a given patch are determined up to a rigid transform, it is clear that the above problem is equivalent to the one where the measurements are
\begin{equation}
\label{noisy_measurements}
x_{k,i} = \bar{x}_k + \epsilon_{k,i}, \qquad \|\epsilon_{k,i}\|\leq \varepsilon \quad (k,i) \in E(\G).
\end{equation}
By equivalent, we mean that the reconstruction errors obtained using either \eqref{noisy_measurements'} or \eqref{noisy_measurements} are equal.  The reason we use the latter measurements is that
the analysis in this case is much more simple.

The reconstruction error is defined as follows. Generally, let $\Zs$ be the output of Algorithms \ref{algo:SPEC} and \ref{algo:GRET} using \eqref{noisy_measurements} as input, and let
\begin{equation}
\label{xt}
Z_0 \defn [\bar x_1 \cdots \bar x_N \ 0 \cdots 0] \in \mathbb{R}^{d \times (N+M)},
\end{equation}
where we assume that the centroid of $\{\bar x_1, \cdots, \bar x_N\}$ is at the origin.

Ideally, we would require that $\Zs = Z_0$ (up to a rigid transformation) when there is no noise, that is,
when $\varepsilon=0$. This is the exact recovery phenomena that we considered earlier. In general, the gap between $Z_0$ and $\Zs$ is a measure of the reconstruction quality. Therefore, we define the
reconstruction error to be
\begin{equation*}
\eta=\min_{\Theta \in \mathbb{O}(d)} \ \lVert \Zs - \Theta Z_0 \rVert_F.
\end{equation*}
Note that we are not required to factor out the translation since $Z_0$ is centered by construction.

Our main results are the following.
\begin{theorem}[Stability of \texttt{GRET-SPEC}]
\label{thm:stability1}
Assume that $R$ is the radius of the smallest Euclidean ball that encloses the clean configuration $\{\bar x_1,\ldots,\bar x_N\}$. For fixed noise level $ \varepsilon \geq 0$ and membership graph $\Gamma$, suppose we input  the noisy measurements \eqref{noisy_measurements} to \texttt{GRET-SPEC}. If $\mathrm{rank}(C_0)=(M-1)d$, then we have the following bound for \texttt{GRET-SPEC}$:$
\begin{equation*}
\eta \leq  \frac{ |E(\Gamma)|^{1/2}}{\lambda_2(L)} (K_1 \varepsilon + K_2 \varepsilon^2),
\end{equation*}
where
\begin{equation*}
K_1 =   \frac{8    \pi R}{\mu_{d+1}(C)}     \sqrt{2MN|E(\G)|(2+N)d(d+1)} \left( 4R \frac{ \sqrt{N |E (\G)|}} {\lambda_2(L) }  + 1 \right)+  \sqrt{2+N+M} .
\end{equation*}
and
\begin{equation*}
K_2 =   \frac{8    \pi R}{\mu_{d+1}(C)}   \sqrt{2MN |E(\G)| (2+N)d(d+1)}  \left(2 \frac{ \sqrt{N \lvert E(\G) \rvert} }{\lambda_2(L)} + 1 \right) .
\end{equation*}
Here $\lambda_2(L)$ is the second smallest eigenvalue of $L$.
\end{theorem}

We assume here that $\mu_{d+1}(C) $ is non-zero\footnote{Numerical experiments suggest that this is indeed the case if $\mathrm{rank}(C_0)=(M-1)d$. In fact, we notice a growth in the eigenvalue with the increase in noise level. We have however not been able to prove this fact.}. 
The bounds here are in fact quite loose. Note that when $\varepsilon = 0$, we recover the exact recovery result for \texttt{GRET-SPEC} provided in  \cite{zha2009spectral,gortler2013affine}.

\begin{theorem}[Stability of \texttt{GRET-SDP}]
\label{thm:stability2}
Under the conditions of Theorem \ref{thm:stability1}, we have the following for \texttt{GRET-SDP}$:$
\begin{equation*}
\eta \leq \frac{|E(\Gamma)|^{1/2}}{\lambda_2(L)}  \left[ 32 \sqrt{ 2d(d+1)(2+N) |E(\Gamma)|} \mu^{-1/2}_{d+1}(C_0) R  + \sqrt{2+N+M} \right] \varepsilon.
\end{equation*}
\end{theorem}
The bounds  are again quite loose. The main point here is that the reconstruction error for \texttt{GRET-SDP} is within a constant factor of the noise level.
In particular, Theorem \ref{thm:stability2} subsumes the exact recovery condition  $\mathrm{rank}(C_0)=(M-1)d$ described in Section \ref{ER}.

The rest of this Section is devoted to the proofs of Theorem \ref{thm:stability1} and \ref{thm:stability2}. First, we introduce some notations.

\textbf{Notations}. Note that the patch-stress matrix in $(\p_1)$ is computed from the noisy measurements \eqref{noisy_measurements}, and the same patch-stress matrix is used in $(\p_2)$. The quantities $\Gs, \Ws, \Os$, and $\Zs$ are as defined in Algorithms  \ref{algo:SPEC} and \ref{algo:GRET} . We continue to denote the clean patch-stress matrix by $C_0$.  Define
\begin{equation*}
O_0 \defn [I_d \cdots I_d] \quad \text{and} \quad G_0 \defn O_0^T O_0.
\end{equation*}
Let $e_1,\ldots,e_d$ be the standard basis vectors of $\mathbb{R}^d$, and let $e$ be the all-ones vector of length $M$. Define
\begin{equation}
\label{si_def}
s_i \defn \frac{1}{\sqrt M} e \otimes e_i \in \mathbb{R}^{Md} \qquad (1 \leq i \leq d).
\end{equation}
Note that every $d \times d$ block of $G_0$ is $I_d$, and that we can write
\begin{equation}
\label{decomp}
G_0 = \sum_{i=1}^d \ M s_i s_i^T.
\end{equation}

We first present an  estimate that applies generally to both algorithms. The proof is provided in Section \ref{proof:prop:basic_bound}.
\begin{proposition}[Basic estimate]
\label{prop:basic bound}
Let $R$ be the radius of the smallest Euclidean ball that encloses the clean configuration. Then, for any arbitrary $\Theta$,
\begin{equation}
\|\Zs - \Theta Z_0\|_F \leq \frac{ |E(\Gamma)|^{1/2}}{\lambda_2(L)}\Big[   R (2+N)^{1/2}  \|\Os - \Theta O_0 \|_F + \varepsilon  (2+N+M)^{1/2} \Big].
\end{equation}
\end{proposition}

In other words, the reconstruction error in either case is controlled by the rounding error:
\begin{equation}
\label{delta_def}
\delta=\min_{\Theta \in \mathbb{O}(d)} \|\Os - \Theta O_0\|_F.
\end{equation}
The rest of this Section is devoted to obtaining a bound on $\delta$ for \texttt{GRET-SPEC} and \texttt{GRET-SDP}. In particular, we will show that $\delta$ is of the order of $\varepsilon$ in either case.
Note that the key difference between the two algorithms arises from the eigenvector rounding, namely the assignment of the ``unrounded'' orthogonal transform $\Ws$ (respectively from the patch-stress matrix and the
optimal Gram matrix). The analysis in going from  $\Ws$ to the rounded orthogonal transform $\Os$, and subsequently to $\Zs$, is however common to both algorithms.

We now bound the error in \eqref{delta_def} for both algorithms. Note that we can generally write
\begin{equation*}
W^\star = \big[\sqrt{\alpha_1} u_1 \cdots \sqrt{\alpha_d} u_d \big]^T,
\end{equation*}
where $u_1,\ldots,u_d$ are orthonormal. In \texttt{GRET-SPEC}, each $\alpha_i = M$, while in \texttt{GRET-SDP} we set $\alpha_i$ using the eigenvalues of $\Gs$.

Our first result gives a control on the quantities obtained using eigenvector rounding in terms of their Gram matrices.
\begin{lemma}[Eigenvector rounding]
\label{lemma:unround bound}
There exist $\Theta \in \mathbb{O}(d)$ such that
\begin{equation*}
\|  \Ws  - \Theta O_0  \|_F  \leq \frac{4}{\sqrt{M}}  \| {\Ws}^T \Ws - G_0 \|_F.
\end{equation*}
\end{lemma}

Next, we use a result by Li \cite{li1995polar} to get a bound on the error after orthogonal rounding.
\begin{lemma}[Orthogonal rounding]
\label{lemma:round bound} For arbitrary $ \Theta \in \mathbb{O}(d)$,
\begin{equation*}
\|O^\star  - \Theta O_0\|_F \leq 2 \sqrt{d+1} \  \| \Ws  - \Theta O_0\|_F.
\end{equation*}
\end{lemma}

The proofs of Lemma  \ref{lemma:unround bound} and \ref{lemma:round bound} are provided in Appendices \ref{proof:lemma:unround bound} and \ref{proof:lemma:round bound}.
At this point, we record a result from \cite{mirsky1960symmetric} which is repeatedly used in the proof of these lemmas and elsewhere.
\begin{lemma}[Mirsky, \cite{mirsky1960symmetric}]
\label{lemma:Mirsky}
Let $\vertiii{\cdot}$ be some unitarily invariant norm, and let $A,B \in \mathbb{R}^{n\times n}$. Then
\begin{equation*}
\vertiii{ \ \mathrm{diag}( \sigma_1(A)- \sigma_1(B) ,\cdots, \sigma_n(A)- \sigma_n(B)) \ } \leq \vertiii{A - B}.
\end{equation*}
In particular, the above result holds for the Frobenius and spectral norms.
\end{lemma}

By combining Lemma  \ref{lemma:unround bound} and \ref{lemma:round bound}, we have the following bound for \eqref{delta_def}:
\begin{equation}
\label{delta_bound}
\delta \leq  8 \sqrt{\frac{d+1 }{M} } \ \|{\Ws}^T \Ws - G_0\|_F.
\end{equation}

We now bound the quantity on the right in \eqref{delta_bound} for \texttt{GRET-SPEC} and \texttt{GRET-SDP}.

\subsection{Bound for \texttt{GRET-SPEC}} For the spectral relaxation, this can be done using the Davis-Kahan theorem \cite{bhatia1997matrix}. Note that from \eqref{O_spec}, we can write
\begin{equation}
\label{proj_diff}
\frac{1}{M} ({\Ws}^T\Ws - G_0) =\sum_{i=1}^{d} \! \!  \!r_i r_i^T-   \sum_{j=1}^d  s_j s_j^T .
\end{equation}
Following \cite[Ch. 7]{bhatia1997matrix}, let $A$ be some symmetric matrix and $S$ be some subset of the real line. Denote $P_A(S)$ to be the orthogonal projection onto the subspace spanned by the eigenvectors of $A$ whose eigenvalues are in $S$. A particular implication of the Davis-Kahan theorem is that
\begin{equation}
\label{DK}
\|  \ P_{A}(S_1) - P_{B}(S_2) \ \|_{\mathrm{sp}} \leq \frac{\pi}{2 \rho(S^c_1,S_2)} \| A - B \|_{\mathrm{sp}},
\end{equation}
where $S_1^c$ is the complement of $S_1$, and $ \rho(S_1,S_2)=\min \{|u-v| : u \in S_1, v \in S_2\}$.

In order to apply \eqref{DK} to \eqref{proj_diff}, set $A=C, B=C_0, S_1 = [\mu_1(C),  \ \mu_d(C)],$ and $S_2=\{0\}$.
If $\mathrm{rank}(C_0)=(M-1)d$, then $P_{B}(S_2) = \sum_{j=1}^d  s_j s_j^T$. Applying \eqref{DK}, we get
\begin{equation}
\label{WW_G0}
\|{\Ws}^T\Ws - G_0 \|_{\mathrm{sp}}  \leq \frac{M \pi}{2\mu_{d+1}(C)} \| C -C_0\|_F.
\end{equation}
Now, it is not difficult to verify that for the noise model \eqref{noisy_measurements},
\begin{equation}
\label{C_diff}
\| C -C_0\|_F \leq 2\sqrt{N \lvert E (\G)\rvert} \left[ \Big( 4R\frac{ \sqrt{N \lvert E(\G) \rvert}}{\lambda_2(L) }  + 1\Big ) \varepsilon + \Big(2 \frac{ \sqrt{N \lvert E(\G) \rvert} }{\lambda_2(L)} + 1 \Big)\varepsilon^2 \right].
\end{equation}
Combining Proposition \ref{prop:basic bound} with \eqref{delta_bound},\eqref{WW_G0}, and \eqref{C_diff}, we arrive at Theorem \ref{thm:stability1}.

\subsection{Bound for \texttt{GRET-SDP}} To analyze the bound for \texttt{GRET-SDP}, we require further notations. Recall \eqref{si_def}, and let $S$ be the space spanned by $\{s_1,\ldots,s_d\} \subset \mathbb{R}^{Md}$, and let $\bar S$ be the orthogonal
complement of $S$ in $\mathbb{R}^{Md}$. In the sequel, we will be required to use matrix spaces arising from tensor products of vector spaces. In particular, given two subspaces $U$ and $V$ of $\mathbb{R}^{Md}$, denote by $U \otimes V$ the space spanned by the rank-one matrices $\{uv^T  :  u \in U, v \in V\}$.
In particular, note that $G_0$ is in $S \otimes S$.

Let $A \in \mathbb{R}^{Md \times Md}$ be some arbitrary matrix. We can decompose it into
\begin{equation}
\label{PQT}
A = P + Q + T
\end{equation}
where
\begin{equation*}
P \in S\otimes S, \ Q \in (S\otimes \overline{S}) \cup (\overline{S}\otimes S), \ \text{and} \ T \in \overline{S}\otimes \overline{S}.
\end{equation*}
We  record a result about this decomposition from Wang and Singer \cite{LUDLanhui}.
\begin{lemma} [\cite{LUDLanhui}, pg. 7]
\label{lemma:Lanhui}
Suppose $G_0 + \Delta \succeq 0$ and $\Delta_{ii} = 0 \ (1 \leq i \leq M)$. Let $\Delta=P+Q+T$ as in \eqref{PQT}. Then
\begin{equation*}
T \succeq  0, \quad  \text{and} \quad P_{ij} = - \frac{1}{M}\sum_{l=1}^M T_{ll} \quad (1 \leq i,j \leq M).
\end{equation*}
\end{lemma}

It is not difficult to verify that $\Tr(C_0 G_0)=0$ and that $C_0 \succeq 0$. From \eqref{decomp}, we have
\begin{equation*}
0 =  \Tr(C_0G_0) = \sum_{i=1}^d s_i^T C_0 s_i \geq 0.
\end{equation*}
Since each term in the above sum is non-negative, $C_0 s_i = 0$ for $1 \leq i \leq d$. In other words, $S$ is contained in the null space of $C_0$.
Moreover, if  $\mathrm{rank}(C_0)=(M-1)d$, then $S$ is exactly the null space of $C_0$.
Based on this observation, we give a bound on the residual $T$.

\begin{proposition}[Bound on the residual]
\label{prop:trace formula}
Suppose that $\mathrm{rank}(C_0)=(M-1)d$. Decompose $\Delta = P + Q + T$ as in \eqref{PQT}. Then
\begin{equation}
\Tr(T) \leq 4 \mu_{d+1}^{-1}(C_0) |E(\Gamma)| \varepsilon^2.
\end{equation}
\end{proposition}

\begin{proof}  The main idea here is to compare the objective in $(\p_0)$ with the trace of $T$. To do so, we introduce the following notations. Let $\lambda_1,  \cdots, \lambda_{Md}$ be the full set of eigenvalues of $\Gs$ sorted in non-increasing order, and $q_1,\ldots,q_{Md}$
be the corresponding eigenvectors. Define
\begin{equation*}
\Oss \defn \big[ \sqrt{\lambda_1} q_1 \  \cdots \ \sqrt{\lambda_{Md}} q_{Md} \big]^T \in \mathbb{R}^{Md \times Md},
\end{equation*}
and $\Oss_i$ to be the $i$-th $Md \times d$ block of $\Oss$, that is, $\Oss \defn [ \Oss_1 \ \cdots \  \Oss_M ]$.

By construction, $\Gs={\Oss}^T\Oss$. Moreover, by feasibility,
\begin{equation*}
 G_{ii}^{\star}={\Oss_i}^T \Oss_i = I_d \quad (1 \leq i \leq M).
\end{equation*}
Thus the $d$ columns of $\Oss_i$ form an orthonormal system in $\mathbb{R}^{Md}$. Now define
\begin{equation*}
\Zss \defn \Oss B L^{\dagger} \in \mathbb{R}^{Md \times (N+M)}.
\end{equation*}
 In particular, we will use the fact that $(\Zss, \Oss)$ are the minimizers of the unconstrained program
\begin{equation}
\label{Q0}
\min_{(Z,O)}  \sum_{(k,i)\in E(\Gamma)} \|Z e_{ki} - O_i x_{k,i}\|^2 \quad \text{s.t.} \quad Z \in \mathbb{R}^{Md \times (N+M)}, \ O \in \mathbb{R}^{Md \times Md}.
\end{equation}

Note that $\Tr(C_0 \Gs)=\Tr(C_0 (G_0+\Delta))= \Tr(C_0 T)$. Now, by Lemma \eqref{lemma:Lanhui}, $T \succeq 0$. Therefore, writing
\begin{equation*}
 T = \sum_i v_i v_i^T \quad (v_i \in \bar{S}),
\end{equation*}
we get
\begin{equation*}
\Tr(C_0 T) = \sum_i v_i^T C_0 v_i \geq  \mu_{d+1}(C_0) \sum_i v_i^T  v_i =  \mu_{d+1}(C_0) \Tr(T).
\end{equation*}
Therefore,
\begin{equation}
\label{tr_comp}
 \Tr(T) \leq \mu_{d+1}^{-1}(C_0) \ \Tr(C_0 \Gs).
\end{equation}
We are done if we can bound the term on the right. To do so, we note from \eqref{Q0} that
\begin{equation*}
\Tr(C_0 \Gs)=\Tr(C_0 {\Oss}^T \Oss) = \min_{Z \in \mathbb{R}^{Md \times N+M}} \ \sum_{(k,i)\in E(\Gamma)} \| Z e_{ki} - \Oss_i \bar x_k \|^2.
\end{equation*}
Therefore,
\begin{equation*}
\Tr(C_0 \Gs) \leq \sum_{(k,i)\in E(\Gamma)} \|\Zss e_{ki} - \Oss_i \bar x_k \|^2.
\end{equation*}
To bring in the error term, we write
\begin{equation*}
 \Zss e_{ki} - \Oss_i \bar x_k = \Zss e_{ki} - \Oss_i x_{k,i} + \Oss_i \epsilon_{k,i},
\end{equation*}
and use $\lVert x + y \rVert^2 \leq 2(\lVert x \rVert^2 + \lVert y \rVert^2)$ to get
\begin{equation}
\label{bb1}
  \Tr(C_0 \Gs) \leq 2 \sum_{(k,i)\in E}  \| \Zss e_{ki} - \Oss_ix_{k,i}\|^2 +  2|E(\Gamma)|\varepsilon^2.
\end{equation}
Finally, using the optimality of $(\Zss, \Oss)$ for \eqref{Q0}, we have
\begin{equation}
\label{bb2}
\sum_{(k,i)\in E(\Gamma)}  \| \Zss e_{ki} - \Oss_ix_{k,i}\|^2 \leq  \sum_{(k,i)\in E(\Gamma)} \| Z_0 e_{ki} - I_d x_{k,i}\|^2 \leq |E(\Gamma)|\varepsilon^2.
\end{equation}
The desired result follows from \eqref{tr_comp}, \eqref{bb1}, and \eqref{bb2}.
\end{proof}

Finally, we note that $\Tr(T)$ can be used to bound the difference between the Gram matrices.
\begin{proposition}[Trace bound]
\label{prop:trace bound}
$\|{W^\star}^T \Ws -G_0\|_F \leq 2 \sqrt{2Md \Tr(T)} $.
\end{proposition}
\begin{proof}
We will heavily use decomposition \eqref{PQT} and its properties. Let $\Gs = G_0 + \Delta$. By triangle inequality,
\begin{eqnarray*}
\|{W^\star}^T \Ws - G_0\|_F &\leq&  \|\sum_{i=d+1}^{Md} \lambda_{i}(G^\star) \ u_i u_i^T\|_F + \|\Delta \|_F  \\ \nonumber
&=& \| \ \mathrm{diag}(\lambda_{d+1}(G^\star),\ldots,\lambda_{Md}(G^\star)) \ \|_F +\|\Delta \|_F .
\end{eqnarray*}
Moreover, since the bottom eigenvalues of $G_0$ are zero, it follows from Lemma \ref{lemma:Mirsky} that the norm of the diagonal matrix is bounded by $ \|\Delta\|_F$.
Therefore,
\begin{equation}
\label{bound_Delta}
 \|{W^\star}^T \Ws - G_0\|_F \leq 2 \|\Delta\|_F.
\end{equation}
Fix $\{s_{d+1},\ldots,s_{Md}\}$ to be some orthonormal basis of $\bar S$. For arbitrary $A \in \mathbb{R}^{Md}$, let
\begin{equation*}
 A(p,q) = s^T_p A s_q \qquad (1 \leq p,q \leq Md).
\end{equation*}
That is, $(A(p,q))$ are the coordinates of  $A$ in the basis $\{s_1,...,s_d\}\cup\ \{s_{d+1},\ldots,s_{Md}\}$.

Decompose $\Delta = P + Q + T$ as in \eqref{PQT}. Note that $P,Q,$ and $T$ are represented in the above basis as follows: $P$ is supported  on the upper $d \times d$ diagonal block,
$T$ is supported on the lower $(M-1)d\times (M-1)d$ diagonal block, and $Q$ on the off-diagonal blocks. The matrix $G_0$ is diagonal  in this representation.

We can bound $\|P\|_F$ using Lemma \ref{lemma:Lanhui},
\begin{equation}
\label{bound_P}
\|P\|_F^2=M^2\|P_{11}\|_F^2 = \|\sum_{l=1}^M T_{ll}\|_F^2 \leq \Big[\Tr \ \big(\sum_{l=1}^M T_{ll}\big) \Big]^2 = \Tr(T)^2,
\end{equation}
where we have used the properties $T\succeq 0$ and $T_{ll}\succeq 0 \ (1 \leq l \le M)$. In particular,
\begin{equation}
\label{bound_T}
\|T\|_F \leq \Tr(T).
\end{equation}
On the other hand, since $G_0+ \Delta\succeq 0$, we have $ (G_0+ \Delta)(p,q)^2 \leq (G_0+ \Delta)(p,p) (G_0+ \Delta)(q,q)$. Therefore,
\begin{equation*}
\|Q\|_F^2 = 2\sum_{p = 1}^d\sum_ {q=d+1}^{Md}  Q(p,q)^2  \leq 2\sum_{p = 1} ^d (G_0+ \Delta)(p,p) \sum_{q=d+1}^{Md} T(q,q).
\end{equation*}
Notice that $0 = \Tr(\Delta) = \Tr(T) + \Tr(P)$. Therefore,
 \begin{equation}
 \label{bound_Q}
 \|Q\|_F^2 \leq 2Md \Tr(T) - 2 \Tr(T)^2.
 \end{equation}
Combining  \eqref{bound_Delta}, \eqref{bound_P},  \eqref{bound_Q}, and \eqref{bound_T}, we get the desired bound.
\end{proof}

Putting together \eqref{delta_bound} with Propositions \eqref{prop:basic bound},\eqref{prop:trace formula}, and \eqref{prop:trace bound}, we arrive at Theorem
\eqref{thm:stability2}.

\section{Numerical Experiments}
\label{Sim}

We now present some numerical results on multipatch registration using \texttt{GRET-SPEC} and \texttt{GRET-SDP}. In particular, we study the exact recovery and stability properties of the algorithm.
We define the reconstruction error in terms of the root-mean-square deviation (RMSD) given by
\begin{equation}
\label{rmsd}
\mathrm{RMSD} = \min_{\Omega \in \mathbb{O}(d), t \in \mathbb{R}^d} \ \left[\frac{1}{N} \sum_{k=1}^N \ \lVert \Zs_k - \Omega \bar x_k - t \rVert^2 \right]^{1/2}.
\end{equation}
In other words, the RMSD is calculated after registering (aligning) the original and the reconstructed configurations. We use the SVD-based algorithm \cite{arun1987svd} for this purpose. \newline

\textbf{Experiment 1}. We first consider a few examples concerning the registration of three patches in $\mathbb{R}^2$, where we vary $\G$ by controlling the number of points in the intersection
of the patches. We work with the clean data in \eqref{clean_data} and demonstrate exact recovery (up to numerical precision) for different $\Gamma$.

In the left plot in Figure ~\ref{fig:3examples}, we consider a patch system with $N=10$ points.
The points that belong to two or more patches are marked red, while the rest are marked black. The patches taken in the order $P_1,P_2,P_3$ form a lateration in this case.
As predicted by Corollary \ref{rankstress} and Theorem \ref{thm:LaterationRecovery}, the rank of the patch-stress matrix $C_0$ for this system must be $2(3-1)=4$. This is indeed confirmed by our experiment.
We expect  \texttt{GRET-SPEC} and \texttt{GRET-SDP} to recover the exact configuration. Indeed, we get a very small RMSD of the order of 1e-7 in this case. As shown in the figure,
the reconstructed coordinates obtained using \texttt{GRET-SDP}  perfectly match the original ones after alignment.

We next consider the example shown in the center plot in Figure ~\ref{fig:3examples}. The patch system is not laterated in this case, but the rank of $C_0$ is $4$.
Again we obtain a very small RMSD of the order 1e-7 for this example. This example demonstrates that lateration is not necessary for exact recovery.

\begin{figure}[here]
\centering
\includegraphics[width=0.32\linewidth]{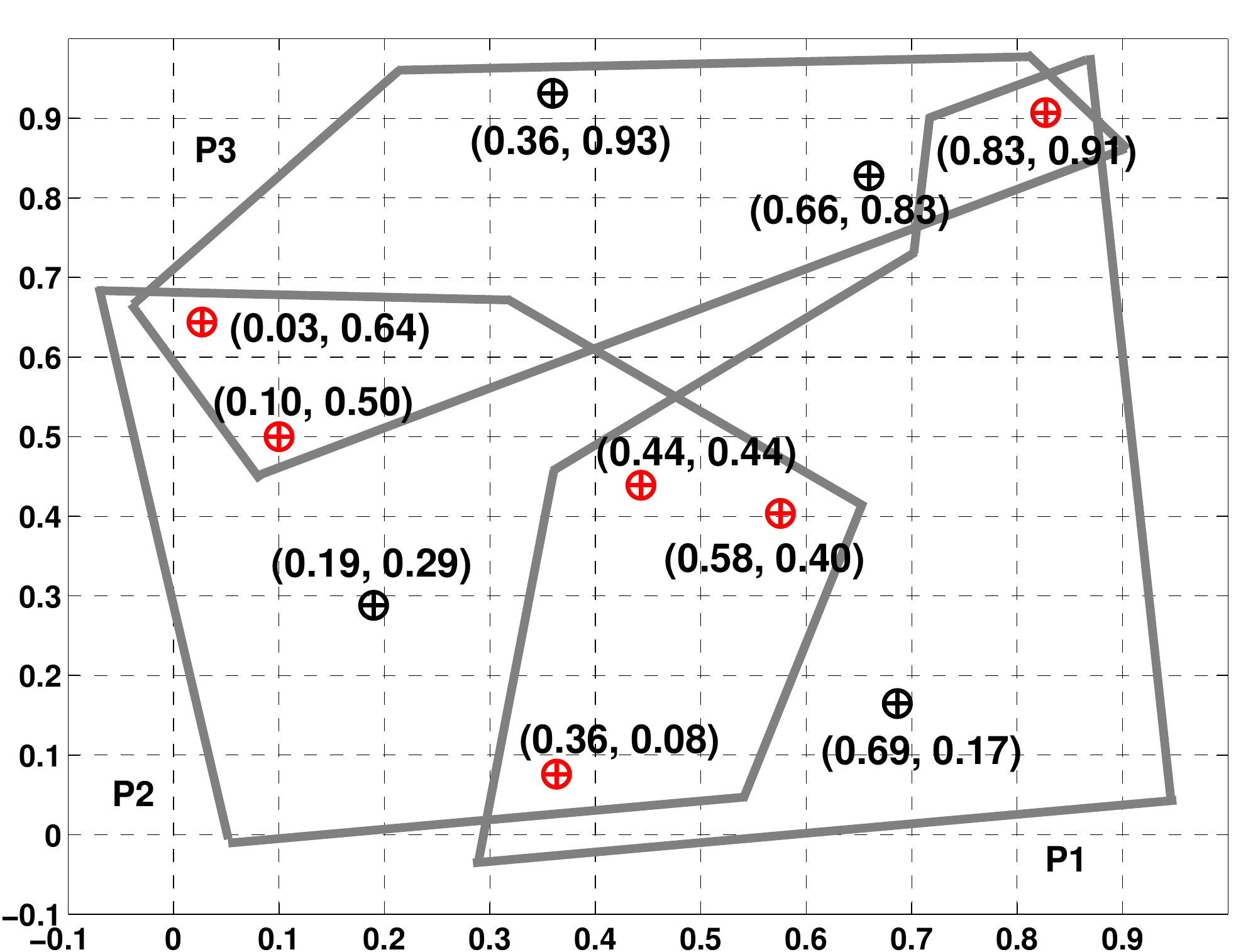}
\includegraphics[width=0.32\linewidth]{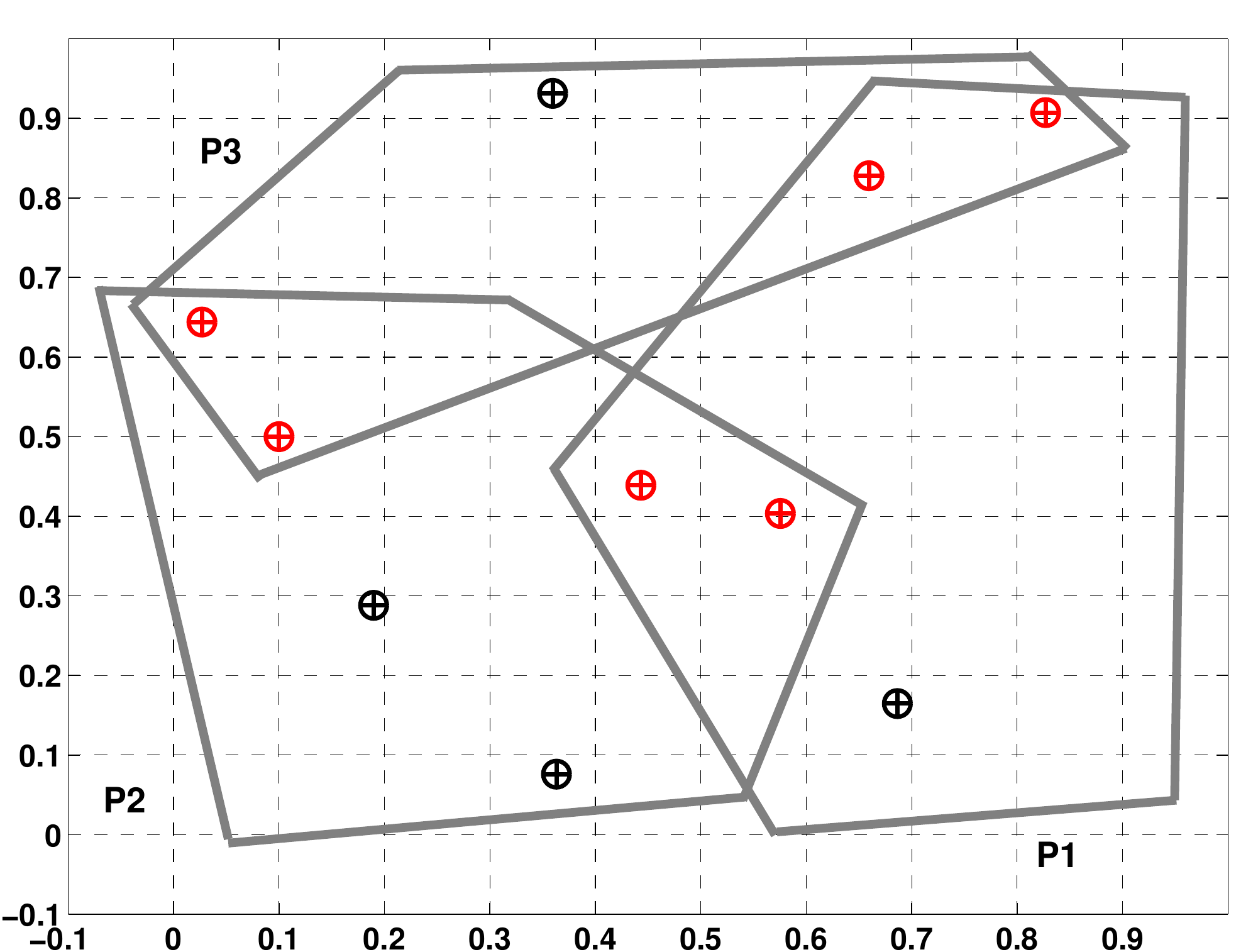}
\includegraphics[width=0.32\linewidth]{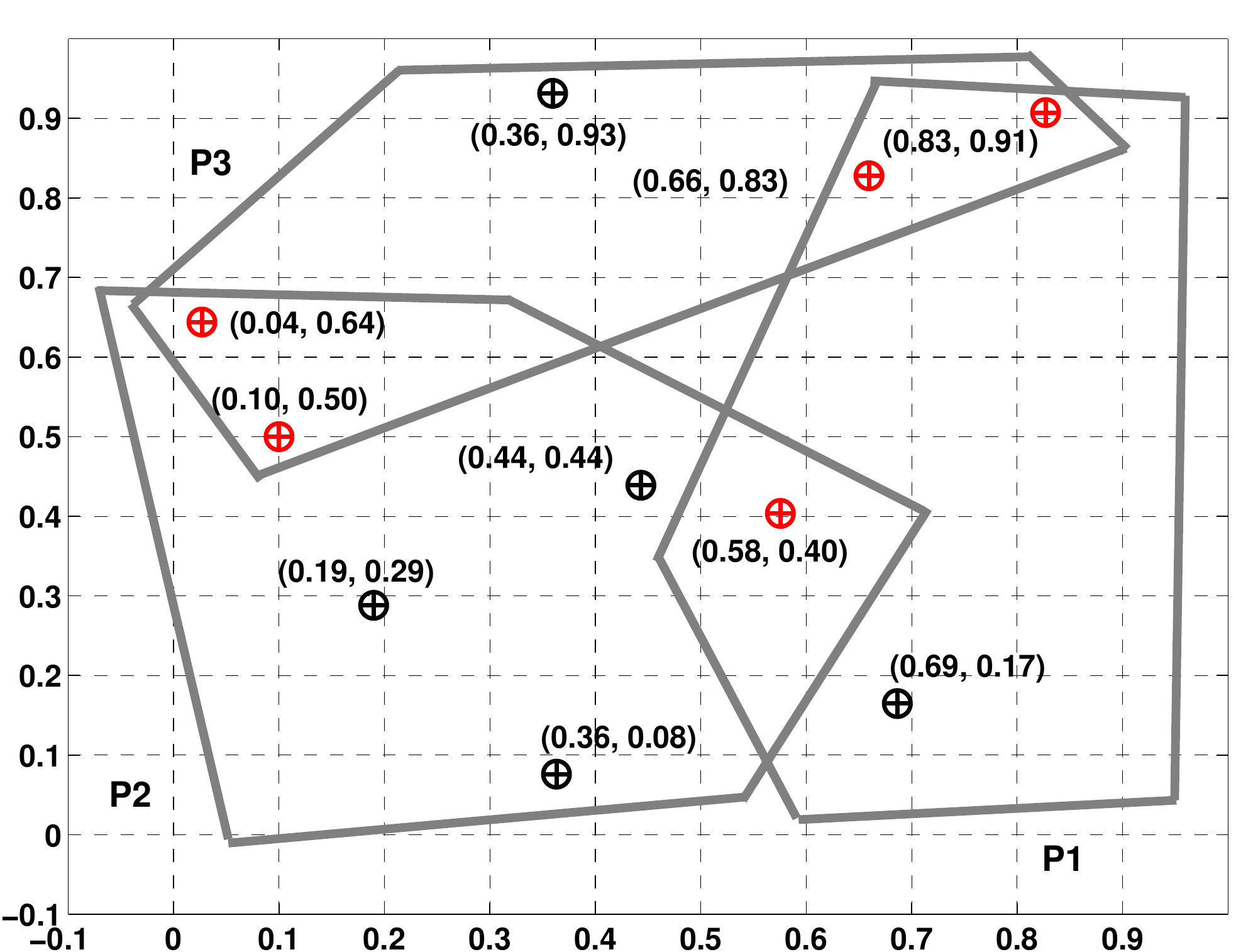}
\caption{Instances of a three-patch systems in $\mathbb{R}^2$. Left: Patch system is laterated. Center: Patch system is not laterated but for which $C_0$ has rank $4$.
Right: The body graph is universally rigid but rank$(C_0) = 3$. The original coordinates are marked with $\circ$, and the  coordinates reconstructed by \texttt{GRET-SDP}
with $+$.}
\label{fig:3examples}
\end{figure}

In the next example, we show that the condition $\rnk(C_0)=(M-1)d$ is not necessary for exact recovery using \texttt{GRET-SDP}. To do so, we use the fact that the
universal rigidity of the body graph
is both necessary and sufficient for exact recovery. Consider the example shown in the right plot in Figure ~\ref{fig:3examples}. This has barely enough points in the patch intersections to make
the  body graph universally rigid.
Experiments confirm that we have exact recovery in this case. However, it can be shown that $\rnk(C_0)<(M-1)d=4$. \\

\begin{figure}[here]
\centering
\includegraphics[width=0.4\linewidth]{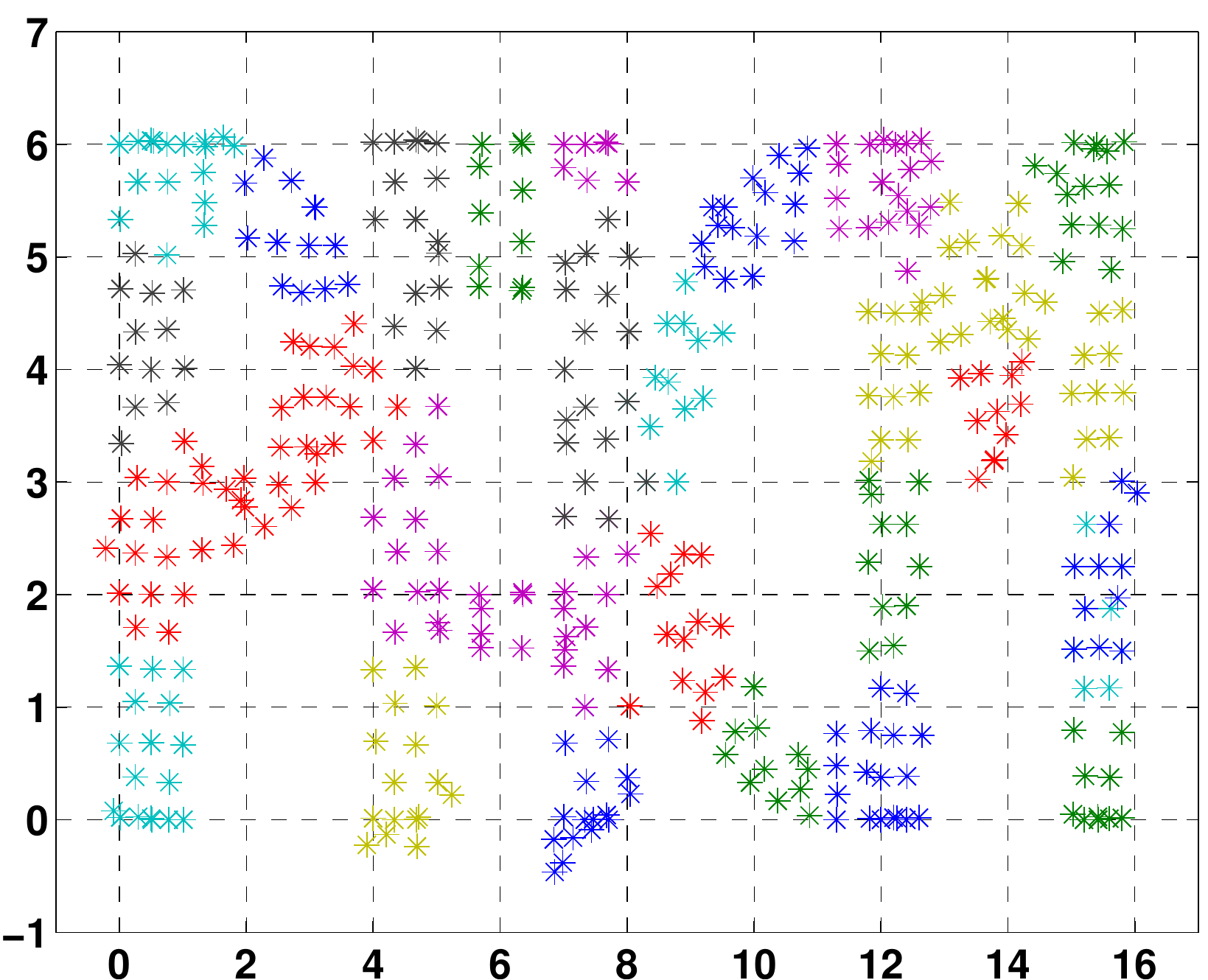}
\includegraphics[width=0.48\linewidth]{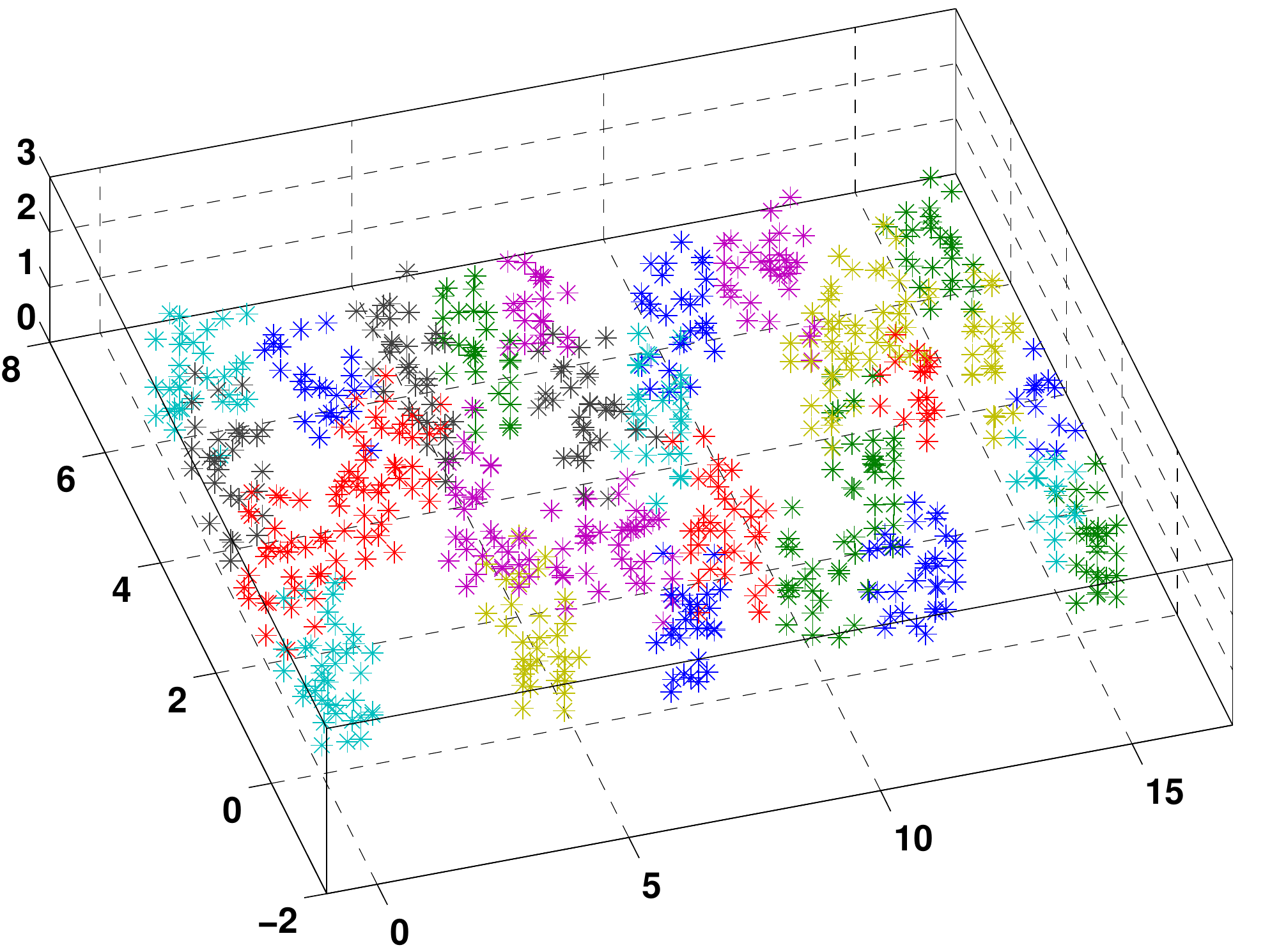}
\caption{Disjoint clusters for the PACM point cloud. Each cluster is marked with a different color. The clusters are augmented to form overlapping patches which are then registered using \texttt{GRET-SDP}.}
\label{fig:pacmcluster}
\end{figure}

\textbf{Experiment 2}.  We now consider the structured PACM data  in $\mathbb{R}^3$ shown in Figure \ref{fig:pacmcluster}. The are a total of $799$ points in this example that are obtained by sampling the 3-dimensional PACM logo \cite{cucuringu2013asap,fang2013disco}.
To begin with, we divide the point cloud into $M=30$ disjoint pieces (clusters) as shown in the figure.
We augment each cluster into a patch by adding points from neighboring clusters. We ensure that there are sufficient common points in the patch system so that $C_0$ has rank $(M-1)d=87$. We generate the measurements using the bounded noise model in \eqref{noisy_measurements}. In particular, we perturb the clean
coordinates using uniform noise over the hypercube $[-\varepsilon,\varepsilon]^d$.
For the noiseless setting, the RMSD's obtained using \texttt{GRET-SPEC} and \texttt{GRET-SDP} are  3.3e-11 and 1e-6.
The respective RMSD's when $\varepsilon = 0.5$ are $1.4743$ and $0.3823$. The results are shown in Figure \ref{fig:pacmrecon}.

\begin{figure}[here]
\centering
\includegraphics[width=0.48\linewidth]{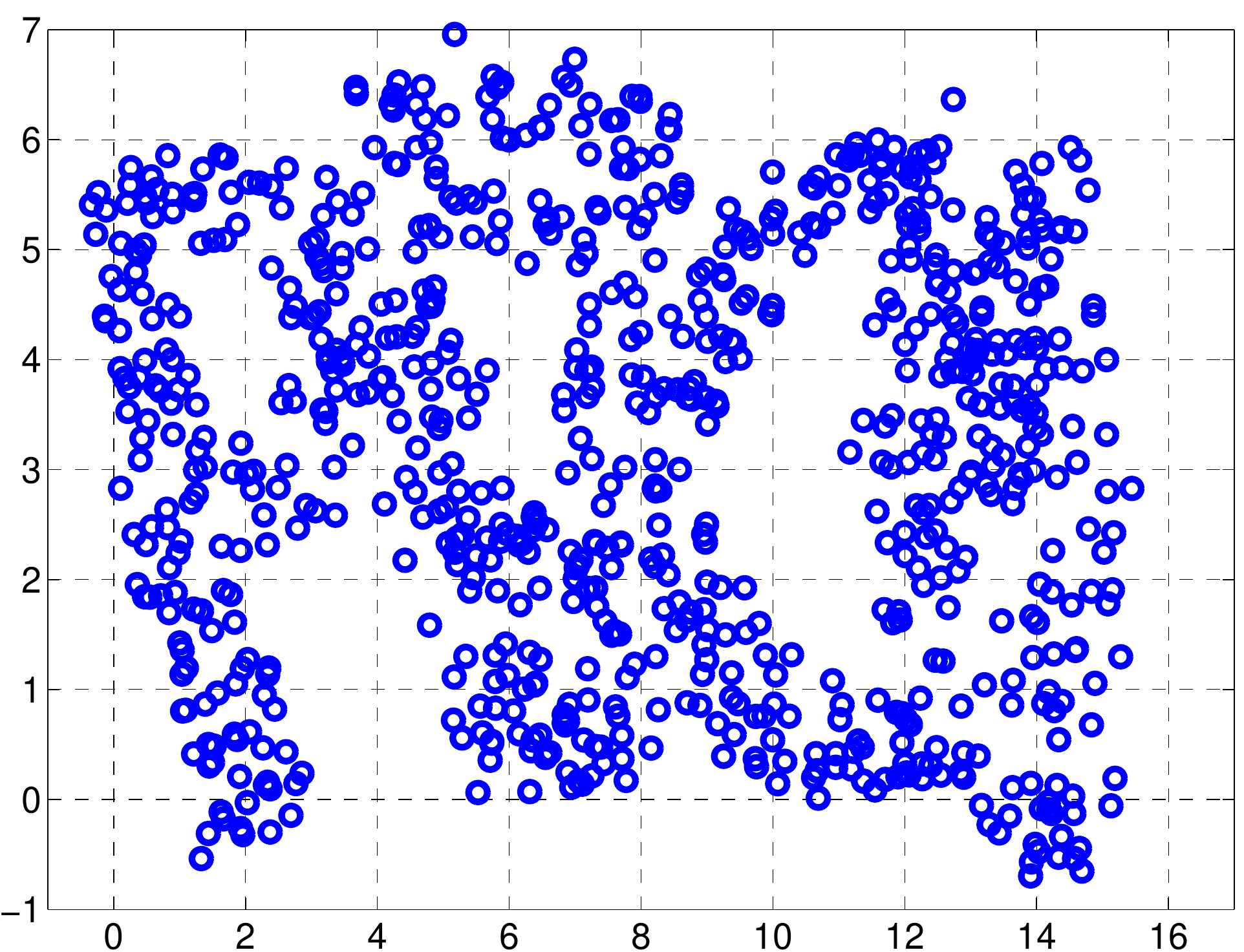}
\includegraphics[width=0.48\linewidth]{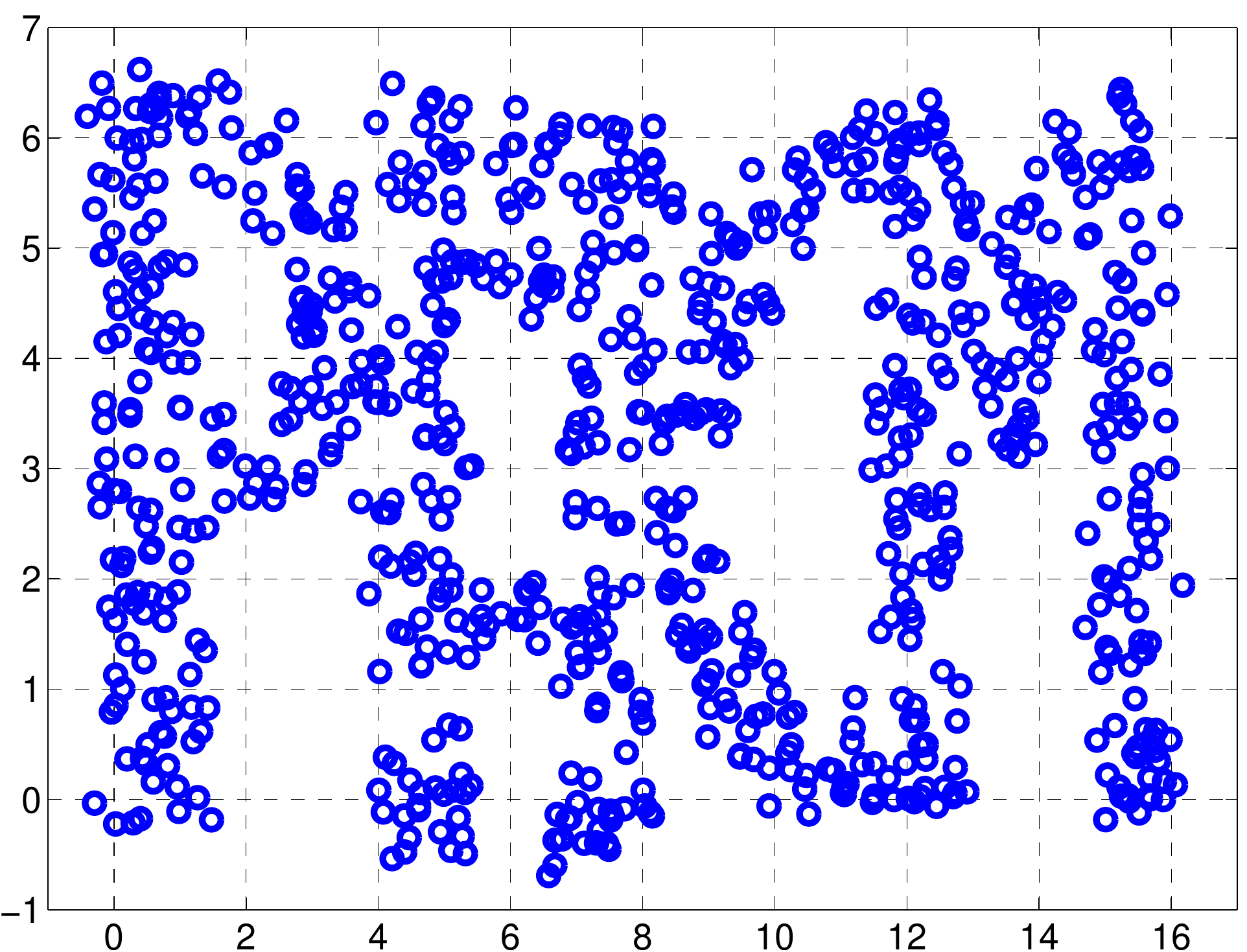}
\caption{Reconstruction of the PACM data from corrupted patch coordinates ($\varepsilon = 0.5$). \textbf{Left}:  \texttt{GRET-SPEC}, $\textrm{RMSD}$ =  1.4743. \textbf{Right}: \texttt{GRET-SDP}, $\textrm{RMSD}$ = 0.3823. The measurements were generated using the noise model in \eqref{noisy_measurements}.}
\label{fig:pacmrecon}
\end{figure}

\textbf{Experiment 3}.  In the final experiment, we demonstrate the stability of \texttt{GRET-SDP} and \texttt{GRET-SPEC} by plotting the RMSD against the noise level for the PACM data.
We use  the noise model in  \eqref{noisy_measurements} and vary $\varepsilon$ from $0$ to $2$ in steps of $0.1$. For a fixed noise level, we average the RMSD over $20$ noise realizations.
The results are reported in the bottom plot in Figure ~\ref{fig:stability1}. We see that the RMSD increases gracefully with the noise level. The result also shows that the semidefinite relaxation is more stable than spectral relaxation, particularly at large noise levels. Also shown in the figure are the RMSD obtained using \texttt{GRET-MANOPT}  with the solutions of \texttt{GRET-SPEC} and \texttt{GRET-SDP} as initialization.
In particular, we used the trust region method provided in the \texttt{Manopt} toolbox \cite{manopt2013} for solving the manifold optimization $(\p_0)$.
For either initialization, we notice  some improvement from the plots. It is clear that the manifold method relies heavily on the initialization, which is not surprising.

\begin{figure}[here]
\centering
\includegraphics[width=0.48\linewidth]{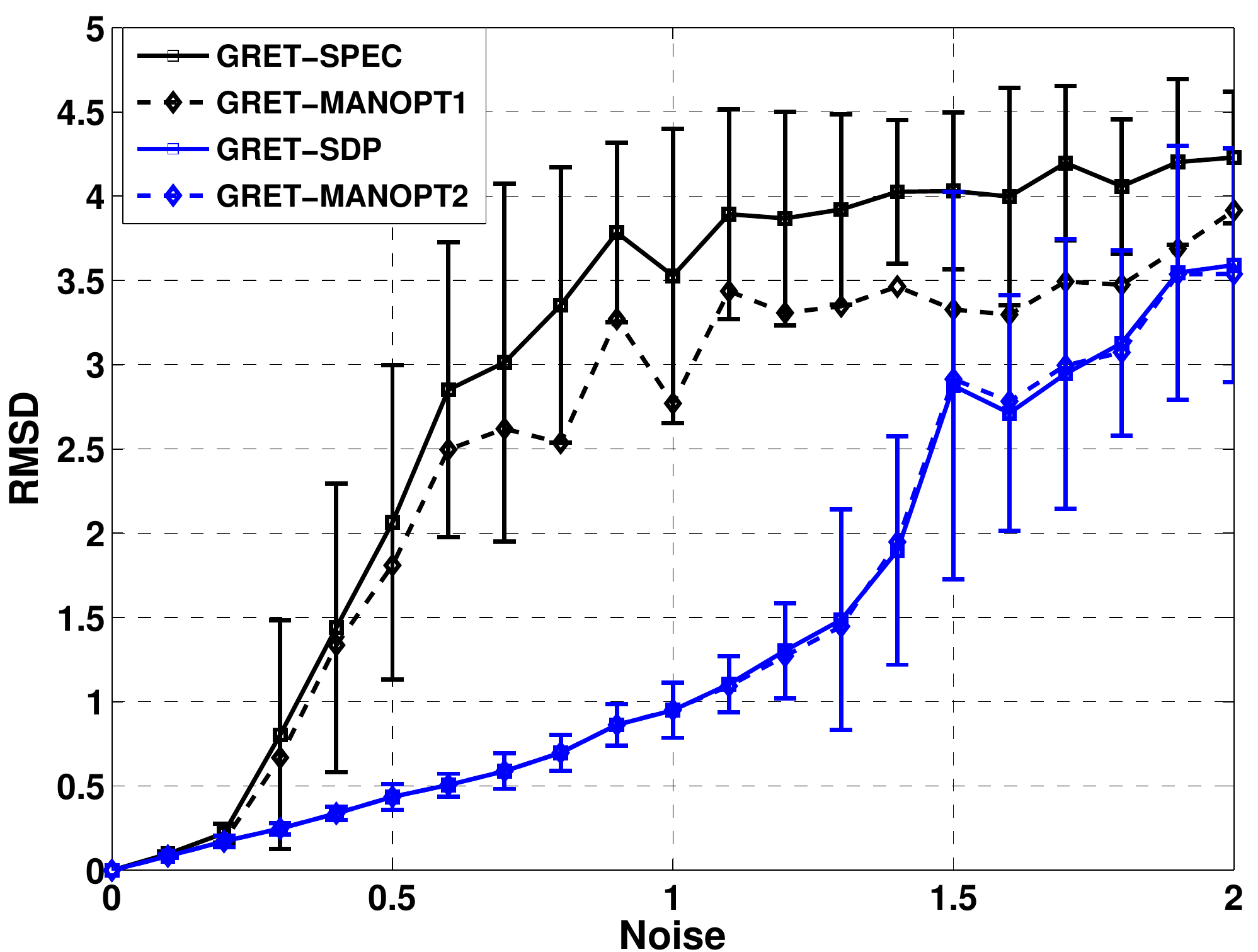}
\includegraphics[width=0.48\linewidth]{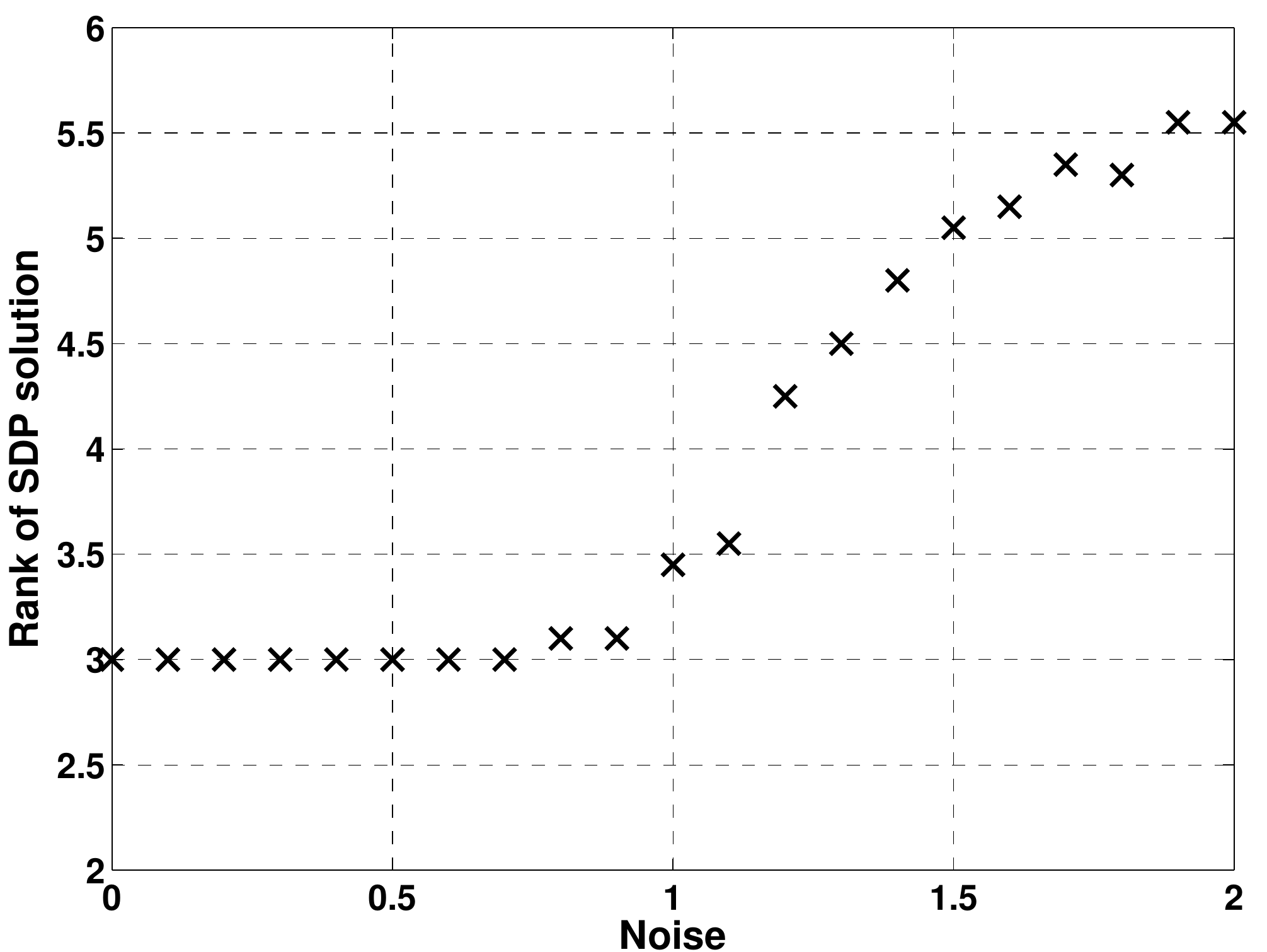}
\caption{\textbf{Left}:  RMSD versus noise level $\varepsilon$. \texttt{GRET-MANOPT1} (resp. \texttt{GRET-MANOPT2}) is the result obtained by refining the output of \texttt{GRET-SPEC} (resp. \texttt{GRET-SDP}) using
manifold optimization. \textbf{Right}: Rank of $\Gs$ in \texttt{GRET-SDP}.}
\label{fig:stability1}
\end{figure}

Finally, we plot the rank of the SDP solution $\Gs$ and notice an interesting phenomenon. Up to a certain noise level, $G^\star$ has the desired rank and rounding is not required.
This means that the relaxation gap is zero for the semidefinite relaxation, and that we can solve the original non-convex problem using \texttt{GRET-SDP} up to a certain noise threshold.
It is therefore not surprising that the RMSD shows no improvement after we refine the SDP solution using manifold optimization. We have noticed that the rank of the SDP solution is stable with
respect to noise for other numerical experiments as well (not reported here).

\section{Discussion}
\label{Conc}

There are several directions along which the present work could be extended and refined. We summarize some of these below.
\begin{enumerate}

\item \textbf{Rank Recovery}. Exhaustive numerical simulations (see, for example, Figure \ref{fig:stability1}) show us that the proposed program is quite stable as far as rank recovery is concerned.
By rank recovery, we mean that $\rnk(G^{\star})=d$. In this case, the relaxation gap is zero -- we have actually solved the original non-convex problem.
We have performed numerical experiments in which we fix some $\Gamma$ for which $\mathrm{rank}(C)=(M-1)d$, and gradually increase the noise in the measurements as per the
model in \eqref{noisy_measurements}.
When the noise is zero, we recover the exact Gram matrix that has rank $d$. What is interesting is that the program keeps returning a rank-$d$ solution up to a certain noise level.
In other words, we observe a \textit{phase transition} phenomenon in which $\rnk(\Gs)$ is consistently $d$ up to a certain noise threshold. This threshold seems to depend on the number of points in the intersection
of the patches, which is perhaps not surprising. A precise understanding of this phase transition in terms of the properties of $\Gamma$ would be an interesting study.

\item \textbf{Conditions on $\G$}. We have seen that the universal rigidity of the body graph (derived from $\G$) is both necessary and sufficient for exact recovery using \texttt{GRET-SDP}. However, to test
unique rigidity, we need to run a semidefinite program \cite{so2007theory}. Unfortunately, the complexity of this program is much more than \texttt{GRET-SDP} itself.
This led us to consider the rank criteria that could be tested efficiently. The rank test is nonetheless not necessary for exact recovery, and
weaker conditions can be found. In particular, an interesting question is whether we could find an efficiently-testable condition that would hold true
for the extreme example in Figure \ref{fig:3examples}, in which $\Gamma$ fails the rank test?

\item \textbf{Tighter Bounds}. The stability in Theorem \ref{thm:stability2} was for the bounded noise model, which made the subsequent analysis quite straightforward.
The goal was to establish that the reconstruction error is within $C\varepsilon$ for some constant $C$ independent of the noise. In particular, the bounds in Theorem \ref{thm:stability2} are quite loose.
One possible direction would be to consider a stochastic noise model with statistically independent perturbations  to tighten the bound.

\item \textbf{Anchor Points}. In sensor network localization, one has to infer the coordinates of sensors from the knowledge of  distances between sensors and its geometric neighbors. In distributed approaches to
sensor localization \cite{cucuringu2012sensor,biswas2008distributed}, one is faced exactly with the multipatch registration problem described in this paper.
Besides the distance information, one often has the added knowledge of the precise positions of selected sensors known as \textit{anchors} \cite{biswas2006snl}. This is often by design and is used to improve the localization
accuracy. The question is can we incorporate the anchor constraints into the present registration algorithm? One possible way of leveraging the existing framework is to introduce an additional
patch (called \textit{anchor patch}) for the anchor points. The anchor coordinates  are assigned to the points in the anchor patch (treating them as local coordinates). This gives us an augmented
bipartite graph $\Gamma_a$ which has one more patch vertex than $\Gamma$, and extra edges connecting the anchor patch
to the anchor vertices. We then proceed exactly as before, that is, we solve for the global coordinates of both the anchor and non-anchor points given the measurements on $\Gamma_a$.
\end{enumerate}

\section*{Acknowledgements}

K. N. Chaudhury was partially supported by the Swiss National Science Foundation under Grant PBELP2-135867 and by Award Number R01GM090200 from the National Institute of General Medical Sciences while he was at Princeton University, and more recently by a Startup Grant from the Indian Institute of Science.

Y. Khoo was partially supported by Award Number R01 GM090200-01 from the National Institute for Health.

A. Singer was partially supported by Award Numbers FA9550-13-1-0076 and FA9550-12-1-0317 from AFOSR, Award Number R01GM090200
from the National Institute of General Medical Sciences, and Award Number LTR DTD 06-05-2012 from the Simons Foundation.

The authors are grateful to the anonymous referees for their comments and suggestions. They particularly thank the referees for pointing out references \cite{zha2009spectral,gortler2013affine}
and their relation to the present work.
The authors also thank Mihai Cucuringu, Lanhui Wang, and Afonso Bandeira for useful discussions, and Nicolas Boumal for advice on the usage of the \texttt{Manopt} toolbox.

\bibliographystyle{siam}
\bibliography{bibref}

\begin{thebibliography}{10}

\bibitem{absil2009optimization}
{\sc P.-A. Absil, R.~Mahony, and R.~Sepulchre}, {\em Optimization {A}lgorithms
  on {M}atrix {M}anifolds}, Princeton University Press, 2009.

\bibitem{arun1987svd}
{\sc K.~S. Arun, T.~S. Huang, and S.~D. Blostein}, {\em Least-squares fitting
  of two 3d point sets}, IEEE Transactions on Pattern Analysis and Machine
  Intelligence,  (1987), pp.~698--700.

\bibitem{Ortho-Cut}
{\sc A.~S. Bandeira, C.~Kennedy, and A.~Singer}, {\em Approximating the little
  {G}rothendieck problem over the orthogonal group}, arXiv:1308.5207,  (2013).

\bibitem{bandeira2012cheeger}
{\sc A.~S. Bandeira, A.~Singer, and D.~A. Spielman}, {\em A {C}heeger
  inequality for the graph connection {L}aplacian}, SIAM Journal on Matrix
  Analysis and Applications, 34 (2013), pp.~1611--1630.

\bibitem{becker2011templates}
{\sc S.~R. Becker, E.~J. Cand{\`e}s, and M.~C. Grant}, {\em Templates for
  convex cone problems with applications to sparse signal recovery},
  Mathematical Programming Computation, 3 (2011), pp.~165--218.

\bibitem{besl1992method}
{\sc P.~J. Besl and N.~D. McKay}, {\em A method for registration of 3d shapes},
  IEEE Transactions on Pattern Analysis and Machine Intelligence, 14 (1992),
  pp.~239--256.

\bibitem{bhatia1997matrix}
{\sc R.~Bhatia}, {\em Matrix {A}nalysis}, vol.~169, Springer, 1997.

\bibitem{biswas2006snl}
{\sc P.~Biswas, T.-C. Liang, K.-C. Toh, Y.~Ye, and T.-C. Wang}, {\em
  Semidefinite programming approaches for sensor network localization with
  noisy distance measurements}, IEEE Transactions on Automation Science and
  Engineering, 3 (2006), pp.~360--371.

\bibitem{biswas2008distributed}
{\sc P.~Biswas, K.-C. Toh, and Y.~Ye}, {\em A distributed {SDP} approach for
  large-scale noisy anchor-free graph realization with applications to
  molecular conformation}, SIAM Journal on Scientific Computing, 30 (2008),
  pp.~1251--1277.

\bibitem{manopt2013}
{\sc N.~Boumal, B.~Mishra, P.-A. Absil, and R.~Sepulchre}, {\em Manopt: {A}
  {M}atlab toolbox for optimization on manifolds}, The Journal of Machine
  Learning Research,  (2014).
\newblock Accepted for publication.

\bibitem{burer2003nonlinear}
{\sc S.~Burer and R.~D.~C. Monteiro}, {\em A nonlinear programming algorithm
  for solving semidefinite programs via low-rank factorization}, Mathematical
  Programming, 95 (2003), pp.~329--357.

\bibitem{candes2009exact}
{\sc E.~Candes and B.~Recht}, {\em Exact matrix completion via convex
  optimization}, Foundations of Computational Mathematics, 9 (2009),
  pp.~717--772.

\bibitem{candes2012phaselift}
{\sc E.~J. Candes, T.~Strohmer, and V.~Voroninski}, {\em Phaselift: {E}xact and
  stable signal recovery from magnitude measurements via convex programming},
  Communications on Pure and Applied Mathematics, 66 (2013), pp.~1241--1274.

\bibitem{chung1997spectral}
{\sc F.~R.~K. Chung}, {\em Spectral {G}raph {T}heory}, vol.~92, American
  Mathematical Society, 1997.

\bibitem{connelly1982rigidity}
{\sc R.~Connelly}, {\em Rigidity and energy}, Inventiones Mathematicae, 66
  (1982), pp.~11--33.

\bibitem{cucuringu2012sensor}
{\sc M.~Cucuringu, Y.~Lipman, and A.~Singer}, {\em Sensor network localization
  by eigenvector synchronization over the euclidean group}, ACM Transactions on
  Sensor Networks, 8 (2012), p.~19.

\bibitem{cucuringu2013asap}
{\sc M.~Cucuringu, A.~Singer, and D.~Cowburn}, {\em Eigenvector
  synchronization, graph rigidity and the molecule problem}, Information and
  Inference, 1 (2012), pp.~21--67.

\bibitem{edelman1998geometry}
{\sc A.~Edelman, T.~A. Arias, and S.~T. Smith}, {\em The geometry of algorithms
  with orthogonality constraints}, SIAM Journal on Matrix Analysis and
  Applications, 20 (1998), pp.~303--353.

\bibitem{fan1955some}
{\sc K.~Fan and A.~J. Hoffman}, {\em Some metric inequalities in the space of
  matrices}, Proceedings of the American Mathematical Society, 6 (1955),
  pp.~111--116.

\bibitem{fang2013disco}
{\sc X.~Fang and K.-C. Toh}, {\em Using a distributed {SDP} approach to solve
  simulated protein molecular conformation problems}, in Distance Geometry,
  Springer, 2013, pp.~351--376.

\bibitem{faugeras1986representation}
{\sc O.~D. Faugeras and M.~Hebert}, {\em The representation, recognition, and
  locating of 3d objects}, International Journal of Robotics Research, 5
  (1986), pp.~27--52.

\bibitem{garber2012almost}
{\sc D.~Garber and E.~Hazan}, {\em Approximating semidefinite programs in
  sublinear time}, in Advances in Neural Information Processing Systems, 2011,
  pp.~1080--1088.

\bibitem{goemans1995improved}
{\sc M.~X. Goemans and D.~P. Williamson}, {\em Improved approximation
  algorithms for maximum cut and satisfiability problems using semidefinite
  programming}, Journal of the ACM, 42 (1995), pp.~1115--1145.

\bibitem{Golub}
{\sc G.~H. Golub and C.~F.~V. Loan}, {\em Matrix {C}omputations}, Johns Hopkins
  University Press, 1996.

\bibitem{gortler2013affine}
{\sc S.~Gortler, C.~Gotsman, L.~Liu, and D.~Thurston}, {\em On affine
  rigidity}, Journal of Computational Geometry, 4 (2013), pp.~160--181.

\bibitem{gortler2010characterizing}
{\sc S.~J. Gortler, A.~D. Healy, and D.~P. Thurston}, {\em Characterizing
  generic global rigidity}, American Journal of Mathematics, 132 (2010),
  pp.~897--939.

\bibitem{gortler2009characterizing}
{\sc S.~J. Gortler and D.~P. Thurston}, {\em Characterizing the universal
  rigidity of generic frameworks}, arXiv preprint arXiv:1001.0172,  (2009).

\bibitem{gower2004procrustes}
{\sc J.~C. Gower and G.~B. Dijksterhuis}, {\em Procrustes {P}roblems}, vol.~3,
  Oxford University Press Oxford, 2004.

\bibitem{grant2008cvx}
{\sc M.~Grant, S.~Boyd, and Y.~Ye}, {\em {CVX}: {M}atlab software for
  disciplined convex programming}, 2008.

\bibitem{helmberg1996interior}
{\sc C.~Helmberg, F.~Rendl, R.~J. Vanderbei, and H.~Wolkowicz}, {\em An
  interior-point method for semidefinite programming}, SIAM Journal on
  Optimization, 6 (1996), pp.~342--361.

\bibitem{hendrickson1992conditions}
{\sc B.~Hendrickson}, {\em Conditions for unique graph realizations}, SIAM
  Journal on Computing, 21 (1992), pp.~65--84.

\bibitem{higham1986}
{\sc J.~N. Higham}, {\em Computing the polar decomposition - with
  applications}, SIAM Journal on Scientific and Statistical Computing, 7
  (1986), pp.~1160--1174.

\bibitem{ho2005pseudo}
{\sc N.-D. Ho and P.~V. Dooren}, {\em On the pseudo-inverse of the {L}aplacian
  of a bipartite graph}, Applied Mathematics Letters, 18 (2005), pp.~917--922.

\bibitem{horn1987closed}
{\sc B.~K.~P. Horn}, {\em Closed-form solution of absolute orientation using
  unit quaternions}, JOSA A, 4 (1987), pp.~629--642.

\bibitem{howard2010estimation}
{\sc S.~D. Howard, D.~Cochran, W.~Moran, and F.~R. Cohen}, {\em Estimation and
  registration on graphs}, arXiv:1010.2983,  (2010).

\bibitem{huang2013consistent}
{\sc Q.-X. Huang and L.~Guibas}, {\em Consistent shape maps via semidefinite
  programming}, in Eurographics Symposium on Geometry Processing, vol.~32,
  2013, pp.~177--186.

\bibitem{Montanari2012}
{\sc A.~Javanmard and A.~Montanari}, {\em Localization from incomplete noisy
  distance measurements}, Foundations of Computational Mathematics, 13 (2013),
  pp.~297--345.

\bibitem{journee2010low}
{\sc M.~Journee, F.~Bach, P.-A. Absil, and R.~Sepulchre}, {\em Low-rank
  optimization on the cone of positive semidefinite matrices}, SIAM Journal on
  Optimization, 20 (2010), pp.~2327--2351.

\bibitem{keller1975closest}
{\sc J.~B. Keller}, {\em Closest unitary, orthogonal and hermitian operators to
  a given operator}, Mathematics Magazine, 48 (1975), pp.~192--197.

\bibitem{GlobalRegKrishnan}
{\sc S.~Krishnan, P.~Y. Lee, J.~B. Moore, and S.~Venkatasubramanian}, {\em
  Global registration of multiple 3d point sets via
  optimization-on-a-manifold}, Eurographics Symposium on Geometry Processing,
  (2005), pp.~187--197.

\bibitem{lerman2012robust}
{\sc G.~Lerman, M.~McCoy, J.~A. Tropp, and T.~Zhang}, {\em Robust computation
  of linear models, or how to find a needle in a haystack}, arXiv:1202.4044,
  (2012).

\bibitem{li1995polar}
{\sc R.-C. Li}, {\em New perturbation bounds for the unitary polar factor},
  SIAM Journal on Matrix Analysis and Applications, 16 (1995), pp.~327--332.

\bibitem{lovasz1991cones}
{\sc L.~Lov{\'a}sz and A.~Schrijver}, {\em Cones of matrices and set-functions
  and 0-1 optimization}, SIAM Journal on Optimization, 1 (1991), pp.~166--190.

\bibitem{mirsky1960symmetric}
{\sc L.~Mirsky}, {\em Symmetric gauge functions and unitarily invariant norms},
  The Quarterly Journal of Mathematics, 11 (1960), pp.~50--59.

\bibitem{mitra2004registration}
{\sc N.~J. Mitra, N.~Gelfand, H.~Pottmann, and L.~Guibas}, {\em Registration of
  point cloud data from a geometric optimization perspective}, in Eurographics
  Symposium on Geometry Processing, 2004, pp.~22--31.

\bibitem{naor2013efficient}
{\sc A.~Naor, O.~Regev, and T.~Vidick}, {\em Efficient rounding for the
  noncommutative {}rothendieck inequality}, in Proceedings of the 45th Annual
  ACM Symposium on Theory of Computing, 2013, pp.~71--80.

\bibitem{nemirovski2007sums}
{\sc A.~Nemirovski}, {\em Sums of random symmetric matrices and quadratic
  optimization under orthogonality constraints}, Mathematical Programming, 109
  (2007), pp.~283--317.

\bibitem{nesterov1998semidefinite}
{\sc Y.~Nesterov}, {\em Semidefinite relaxation and nonconvex quadratic
  optimization}, Optimization methods and software, 9 (1998), pp.~141--160.

\bibitem{pottmann2006geometry}
{\sc H.~Pottmann, Q.-X. Huang, Y.-L. Yang, and S.-M. Hu}, {\em Geometry and
  convergence analysis of algorithms for registration of 3d shapes},
  International Journal of Computer Vision, 67 (2006), pp.~277--296.

\bibitem{ranjan2013incremental}
{\sc G.~Ranjan, Z.-L. Zhang, and D.~Boley}, {\em Incremental computation of
  pseudo-inverse of {L}aplacian: Theory and applications}, arXiv:1304.2300,
  (2013).

\bibitem{rusinkiewicz2001efficient}
{\sc S.~Rusinkiewicz and M.~Levoy}, {\em Efficient variants of the {ICP}
  algorithm}, in Proceedings of the Third International Conference on 3d
  Digital Imaging and Modeling, 2001, pp.~145--152.

\bibitem{Sharp}
{\sc G.~C. Sharp, S.~W. Lee, and D.~K. Wehe}, {\em Multiview registration of 3d
  scenes by minimizing error between coordinate frames}, in Proceedings of the
  7th European Conference on Computer Vision-Part II, Springer-Verlag, 2002,
  pp.~587--597.

\bibitem{singer2011angular}
{\sc A.~Singer}, {\em Angular synchronization by eigenvectors and semidefinite
  programming}, Applied and Computational Harmonic Analysis, 30 (2011),
  pp.~20--36.

\bibitem{MatCompletionSinger}
{\sc A.~Singer and M.~Cucuringu}, {\em Uniqueness of low-rank matrix completion
  by rigidity theory}, SIAM Journal on Matrix Analysis and Applications, 31(4)
  (2010), pp.~1621--1641.

\bibitem{so2011moment}
{\sc A.~M.-C. So}, {\em Moment inequalities for sums of random matrices and
  their applications in optimization}, Mathematical Programming, 130 (2011),
  pp.~125--151.

\bibitem{so2007theory}
{\sc A.~M.-C. So and Y.~Ye}, {\em Theory of semidefinite programming for sensor
  network localization}, Mathematical Programming, 109 (2007), pp.~367--384.

\bibitem{spielman2004nearly}
{\sc D.~A. Spielman and S.-H. Teng}, {\em Nearly-linear time algorithms for
  graph partitioning, graph sparsification, and solving linear systems}, in
  Proceedings of the Thirty-sixth Annual ACM Symposium on Theory of Computing,
  2004, pp.~81--90.

\bibitem{toh1999sdpt3}
{\sc K.-C. Toh, M.~J. Todd, and R.~H. Tutuncu}, {\em {SDPT3} -- {A} {M}atlab
  software package for semidefinite programming, version 1.3}, Optimization
  Methods and Software, 11 (1999), pp.~545--581.

\bibitem{tzeneva2011global}
{\sc T.~Tzeneva}, {\em Global alignment of multiple 3d scans using eigevector
  synchronization}, Senior Thesis, Princeton University (supervised by S.
  Rusinkiewicz and A. Singer),  (2011).

\bibitem{vandenberghe1996semidefinite}
{\sc L.~Vandenberghe and S.~Boyd}, {\em Semidefinite programming}, SIAM review,
  38 (1996), pp.~49--95.

\bibitem{Vishnoi2012}
{\sc N.~K. Vishnoi}, {\em {L}x=b}, Foundations and Trends in Theoretical
  Computer Science, 8 (2012), pp.~1--141.

\bibitem{waldspurger2012phase}
{\sc I.~Waldspurger, A.~d'Aspremont, and S.~Mallat}, {\em Phase recovery,
  maxcut and complex semidefinite programming}, Mathematical Programming,
  (2013), pp.~1--35.

\bibitem{LUDLanhui}
{\sc L.~Wang and A.~Singer}, {\em Exact and stable recovery of rotations for
  robust synchronization}, Information and Inference: A Journal of the IMA, 2
  (2013), pp.~145--193.

\bibitem{wen2012block}
{\sc Z.~Wen, D.~Goldfarb, S.~Ma, and K.~Scheinberg}, {\em Block coordinate
  descent methods for semidefinite programming}, in Handbook on Semidefinite,
  Conic and Polynomial Optimization, Springer, 2012, pp.~533--564.

\bibitem{859274}
{\sc J.~A. Williams and M.~Bennamoun}, {\em Simultaneous registration of
  multiple point sets using orthonormal matrices}, in Proceedings of IEEE
  International Conference on Acoustics, Speech, and Signal Processing, vol.~6,
  2000, pp.~2199--2202.

\bibitem{wolkowicz2002semidefinite}
{\sc H.~Wolkowicz and M.~F. Anjos}, {\em Semidefinite programming for discrete
  optimization and matrix completion problems}, Discrete Applied Mathematics,
  123 (2002), pp.~513--577.

\bibitem{SDP_handbook}
{\sc H.~Wolkowicz, R.~Saigal, and L.~Vandenberghe}, {\em Handbook of
  {S}emidefinite {P}rogramming: {T}heory, {A}lgorithms, and {A}pplications},
  vol.~27, Kluwer Academic Pub, 2000.

\bibitem{zha2009spectral}
{\sc H.~Zha and Z.~Zhang}, {\em Spectral properties of the alignment matrices
  in manifold learning}, SIAM review, 51 (2009), pp.~545--566.

\bibitem{zhang2010arap}
{\sc L.~Zhang, L.~Liu, C.~Gotsman, and S.~Gortler}, {\em An
  {A}s-{R}igid-{A}s-{P}ossible approach to sensor network localization}, ACM
  Transactions on Sensor Networks, 6 (2010), p.~35.

\bibitem{5447068}
{\sc L.~Zhi-Quan, M.~Wing-Kin, A.-C. So, Y.~Ye, and S.~Zhang}, {\em
  Semidefinite relaxation of quadratic optimization problems}, IEEE Signal
  Processing Magazine, 27 (2010), pp.~20--34.

\bibitem{zhu2010universal}
{\sc Z.~Zhu, A.~M.-C. So, and Y.~Ye}, {\em Universal rigidity and edge
  sparsification for sensor network localization}, SIAM Journal on
  Optimization, 20 (2010), pp.~3059--3081.

\end{thebibliography}

\section{Technical proofs}

In this Section, we give the proof of Propositions \ref{prop:Rank_C0} and \ref{prop:basic bound}, and Lemmas  \ref{lemma:unround bound} and \ref{lemma:round bound},.

\subsection{Proof of Proposition \ref{prop:Rank_C0}}
\label{proof:Rank_C0}

We are done if we can show that there exists a bijection between the nullspace of $C$ and that of $C_0$.
To do so, we note that the associated quadratic forms can be expressed as
\begin{equation*}
u^T C u = \min_{z \in \mathbb{R}^{1 \times N+M}} \ \sum_{(k,i) \in E(\Gamma)} \ \lVert z e_{ki} - u_i^T x_{k,i}  \rVert^2,
\end{equation*}
and
\begin{equation*}
v^T C_0 v = \min_{z \in \mathbb{R}^{1 \times N+M}} \ \sum_{(k,i) \in E(\Gamma)} \ \lVert z e_{ki} - v_i^T \bar x_k  \rVert^2.
\end{equation*}
Here $u_1,\ldots,u_M$ are the $d \times 1$ blocks of the vector $u \in \mathbb{R}^{Md \times 1}$.

Now, it follows from \eqref{clean_data} that there is a one-to-one correspondence between $u$ and $v$, namely
\begin{equation*}
u_i = \bar{O}_i v_i \qquad (1 \leq i \leq M),
\end{equation*}
such that $u^T C u = v^T C_0 v$. In other words, the null space of $C$ is related to the null space of $C_0$ through an orthogonal transform, as was required to be shown.

\subsection{Proof of Proposition \ref{prop:basic bound}}
\label{proof:prop:basic_bound}

Without loss of generality, we assume that the smallest Euclidean ball that encloses the clean configuration $\{\bar x_1,\ldots,\bar x_N\}$ is centered at the origin, that is,
\begin{equation}
\label{centering}
\| \bar x_k \| \leq R \qquad (1 \leq k \leq N).
\end{equation}
Let $B_0$ be the matrix $B$ in \eqref{coeff} computed from the clean measurements, i.e., from \eqref{noisy_measurements} with $\varepsilon=0$.
Let $B_0 + H$ be the same matrix obtained from \eqref{noisy_measurements} for some $\varepsilon > 0$.

Recall that $Z_0 = O_0 B_0 L^\dagger$ (by the centering assumption in \eqref{xt}). Therefore,
\begin{equation*}
\|\Zs- \Theta Z_0\|_F  =  \|\Os(B_0+H) L^\dagger - \Theta O_0 B_0 L^\dagger\|_F = \|(\Os - \Theta  O_0) B_0 L^\dagger + \Os H L^\dagger\|_F.
\end{equation*}
By triangle inequality,
\begin{equation}
\label{Zb}
\|\Zs- \Theta Z_0\|_F  \leq \|\Os - \Theta  O_0\|_F  \|B_0 L^\dagger\|_F + \|\Os H L^\dagger\|_F,
\end{equation}
Now
\begin{equation*}
 \|B_0 L^\dagger\|_F \leq  \| L^\dagger \|_{\mathrm{sp}} \|B_0\|_F = \frac{1}{\lambda_2(L)} \|B_0\|_F,
\end{equation*}
where $\lambda_2(L)$ is the smallest non-zero eigenvalue of $L$. On the other hand,
\begin{equation*}
 B_0 = \sum_{(k,i)\in E(\Gamma)} (e^{M}_i \otimes I_d){\bar x_k}e_{ki}^T .
\end{equation*}
Using  Cauchy-Schwarz and \eqref{centering}, we get
\begin{eqnarray*}
 \|B_0\|^2_F &=& \sum_{(k,i)\in E(\Gamma)}  \sum_{(l,j)\in E(\Gamma)} \Tr \left( e_{ki} \bar x_k^T (e^{M}_i \otimes I_d)^T(e^{M}_{j} \otimes I_d)\bar x_{l}e_{lj}^T \right) \\
  &=&         \sum_{(k,i)\in E(\Gamma)}  \sum_{(l,i)\in E(\Gamma)}   \bar x_k^T \bar x_l \ e_{ki}^T  e_{li}. \\
  &\leq &   \sum_{(k,i)\in E(\Gamma)} 2R^2 + \sum_{(k,i)\in E(\Gamma)}  \sum_{(l,i)\in E(\Gamma)} R^2.
 \end{eqnarray*}
Therefore,
\begin{equation}
\label{temp1}
 \|B_0 L^\dagger\|_F \leq  \lambda_2(L)^{-1} \sqrt{2+N} |E(\Gamma)|^{1/2} R.
\end{equation}
As for the other term in \eqref{Zb}, we can write
\begin{equation*}
\|\Os H L^{\dagger}\|_F \leq \|L^\dagger\|_{\mathrm{sp}} \|\Os H \|_F  = \lambda_2(L)^{-1}  \|\Os H\|_F.
\end{equation*}
Now
\begin{equation*}
 \Os H = \Os (B-B_0) = \sum_{(k,i)\in E(\Gamma)} \Os_i \epsilon_{k,i} e_{ki}^T .
\end{equation*}
Therefore, using Cauchy-Schwarz, the orthonormality of the columns of $\Os_i$'s, and the noise model \eqref{noisy_measurements}, we get
\begin{eqnarray*}
\| \Os H \|_F^2  &=& \sum_{(k,i)\in E(\Gamma)}  \sum_{(l,j)\in E(\Gamma)}   (\Os_i \epsilon_{k,i} )^T  (\Os_j \epsilon_{l,j})  e_{ki}^T  e_{li} \\
&\leq&  \sum_{(k,i)\in E(\Gamma)} 2\varepsilon^2 + \sum_{(k,i)\in E(\Gamma)}  \sum_{(l,i)\in E(\Gamma)} \varepsilon^2 + \sum_{(k,i)\in E(\Gamma)}  \sum_{(k,j)\in E(\Gamma)} \varepsilon^2.
\end{eqnarray*}
This gives us
\begin{equation}
\label{temp2}
\|\Os H L^{\dagger}\|_F  \leq \sqrt{2+N+M} |E(\Gamma)|^{1/2} \lambda_2(L)^{-1} \varepsilon.
\end{equation}
Combining \eqref{Zb},\eqref{temp1}, and \eqref{temp2}, we get the desired estimate.

\subsection{Proof of Lemma \ref{lemma:unround bound}}
\label{proof:lemma:unround bound}

The proof is mainly based on the observation that if $u$ and $v$ are unit vectors and $0 \leq u^Tv \leq 1$, then
\begin{equation}
\label{basic_Gram}
\| u - v \| \leq \|uu^T- vv^T\|_F.
\end{equation}
Indeed,
\begin{equation*}
 \|uu^T- vv^T\|_F^2 = \Tr \left(uu^T + vv^T - 2 (u^Tv)^2 \right) \geq  \Tr (uu^T + vv^T - 2 u^Tv) =\| u - v \|^2.
\end{equation*}

To use this result in the present setting, we use the theory of principal angles \cite[Ch. 7.1]{bhatia1997matrix}. This tells us that, for the orthonormal systems $\{u_1,\ldots,u_d\}$ and $\{s_1,\ldots,s_d\}$, we can find $\Omega_1, \Omega_2 \in \mathbb{O}(Md)$ such that
\begin{enumerate}
\item $\Omega_1  [u_1 \cdots u_d] =   [u_1 \cdots u_d] \Theta_1^T$ where  $\Theta_1 \in \mathbb{O}(d)$,
\item $\Omega_2  [s_1 \cdots s_d] =   [s_1 \cdots s_d] \Theta_2^T$ where  $\Theta_2 \in \mathbb{O}(d)$,
\item $(\Omega_1 s_i)^T ( \Omega_2 u_j) = 0$ for $i \neq j$, and $0 \leq  (\Omega_1 s_i)^T( \Omega_2 u_i ) \leq 1$ for $1\leq i \leq d$.
\end{enumerate}
Here $\Theta_1$ and $\Theta_2$ are the orthogonal transforms that map $\{u_1,\ldots,u_d\}$ and $\{s_1,\ldots,s_d\}$ into the corresponding principal vectors.

Using properties $1$ and $2$ and the fact\footnote{To see why the eigenvalues of $\Gs$ are at most $M$ (the authors thank Afonso Bandeira for suggesting this), note that by the SDP constraints, for every block $G_{ij}$,
$$u^TG_{ij}v \leq (\|u\|^2+\|v\|^2)/2 \quad (u,v \in \mathbb{R}^d).$$
Let $x = (x_1,\ldots,x_M)$ where each $x_i  \in \mathbb{R}^d$. Then
$$x^TGx = \sum_{i,j} x_i^TG_{ij}x_j  \leq \sum_{i,j}  (\|x_i\|^2+\|x_j\|^2)/2 = M \|x\|^2.$$
} that $\alpha_i \leq M$, we can write
\begin{equation*}
\sqrt{M}  \  \|  \Theta_1  \Ws - \Theta_2 O_0 \|_F \leq   \| \Omega_1  [\alpha_1 u_1 \cdots \alpha_d u_d] -   M \Omega_2 [ s_1 \cdots  s_d]   \|_F + \Big[ \sum_{i=1}^d (M - \alpha_i )^2 \Big]^{1/2}.
\end{equation*}
Moreover, by triangle inequality,
\begin{equation*}
\|  \Omega_1  [\alpha_1 u_1 \cdots \alpha_d u_d]  - M \Omega_2 [s_1 \cdots s_d]  \|_F \leq M \|  \Omega_1 [u_1 \cdots u_d] -  \Omega_2 [ s_1 \cdots s_d] \|_F + \Big[ \sum_{i=1}^d (M - \alpha_i )^2 \Big]^{1/2}.
\end{equation*}
Therefore,
\begin{equation*}
\label{BB2}
\sqrt{M}  \  \|  \Theta_1  \Ws - \Theta_2 O_0 \|_F  \leq M \|  \Omega_1 [u_1 \cdots u_d] -  \Omega_2 [ s_1 \cdots s_d] \|_F + 2\Big[ \sum_{i=1}^d (M - \alpha_i )^2 \Big]^{1/2}.
\end{equation*}

Now, using \eqref{basic_Gram} and the principal angle property $3$, we get
\begin{equation*}
\| \Omega_1[  u_1 \cdots  u_d] - \Omega_2 [s_1 \cdots s_d]  \|_F \leq \| \sum_{i=1}^d \Omega_1 u_i(\Omega_1 u_i)^T -  \sum_{i=1}^d \Omega_2 s_i ( \Omega_2 s_i )^T \|_F .
\end{equation*}
Moreover, using triangle inequality and properties $1$ and $2$, we have
\begin{equation*}
\label{BB1}
M \| \sum_{i=1}^d \Omega_1 u_i  (\Omega_1 u_i)^T -  \sum_{i=1}^d \Omega_2 s_i ( \Omega_2 s_i )^T \|_F \leq  \|{\Ws}^T \Ws - G_0 \|_F +  \Big[ \sum_{i=1}^d (M - \alpha_i )^2 \Big]^{1/2}.
\end{equation*}
Finally, note that by Lemma \ref{lemma:Mirsky},
\begin{equation}
\label{BB3}
 \Big[ \sum_{i=1}^d (M - \alpha_i )^2 \Big]^{1/2} \leq  \|{\Ws}^T \Ws - G_0 \|_F.
\end{equation}
Combining the above relations, and setting $\Theta = \Theta_1^T \Theta_2$, we arrive at Lemma \ref{lemma:unround bound}.

\subsection{Proof of Lemma \ref{lemma:round bound}}
\label{proof:lemma:round bound}

This is done by adapting the following result by Li \cite{li1995polar}: If $A,B$ are square and non-singular, and if  $\mathcal{R}(A)$ and $\mathcal{R}(B)$ are their orthogonal
rounding (obtained from their polar decompositions  \cite{higham1986}), then
\begin{equation}
\label{round_Li}
\| \mathcal{R}(A) - \mathcal{R}(B) \|_F \leq \frac{2}{\sigma_{\min}(A) + \sigma_{\min}(B)} \|A - B\|_F.
\end{equation}
We recall that if $A=U\Sigma V^T$ is the SVD of $A$, then $\mathcal{R}(A)=UV^T$.

Note that it is possible that some of the blocks of $\Ws$ are singular, for which the above result does not hold. However, the number of such blocks can be controlled
by the global error.
More precisely, let $\mathcal{B} \subset \{1,2,\ldots,M\}$ be the index set such that,  for $i \in \mathcal{B}$, $\|\Ws_i -  \Theta\|_F \geq \beta$. Then
\begin{equation*}
\| \Ws - \Theta O_0\|_F ^2 \geq \sum_{i \in \mathcal{B}} \|\Ws_i -  \Theta\|_F^2  = \lvert \mathcal{B} \rvert \beta^2.
\end{equation*}
This gives a bound on the size of $\mathcal{B}$. In particular, the rounding error for this set can trivially be bounded as
\begin{equation}
\label{BL1}
\sum_{i \in \mathcal{B}} \|O^{\star}_i -  \Theta\|_F ^2 \leq \sum_{i \in \mathcal{B}} 2d = \frac{2 d} {\beta^2} \| \Ws - \Theta O_0\|_F ^2.
\end{equation}
On the other hand, we known that, for $i \in \mathcal{B}^c$, $\|\Ws_i -  \Theta\|_F < \beta$.  From Lemma \ref{lemma:Mirsky}, it follows that
\begin{equation*}
\lvert 1-\sigma_{\min} (\Ws_i)\rvert  \leq \|{\Ws_i} - \Theta\|_{\mathrm{sp}} < \beta.
\end{equation*}
Fix $\beta \leq 1$. Then $\sigma_{\min}(W^{\star}_i)>1-\beta$, and we have from \eqref{round_Li},
\begin{equation}
\label{BL2}
 \|O^{\star}_i -  \Theta\|_F \leq  \frac{2}{2-\beta} \|{W^{\star}_i} -  \Theta\|_F\qquad (i \in \mathcal{B}^c)
\end{equation}
Fixing $\beta=1/\sqrt{2}$ and combining \eqref{BL1} and \eqref{BL2}, we get the desired bound.

\end{document}